\documentclass[letterpaper]{article}
\pdfoutput=1
\usepackage{aaai}
\usepackage{times}
\usepackage{helvet}
\usepackage{courier}
\usepackage{amsmath}
\usepackage{amsthm}
\usepackage{thm-restate}
\usepackage{multirow}
\usepackage{bigdelim}
\usepackage{fmtcount}
\usepackage{enumitem}
\usepackage{needspace}
\usepackage{csvsimple}
\usepackage{xspace}
\usepackage[only,bigsqcap]{stmaryrd}
\SetSymbolFont{stmry}{bold}{U}{stmry}{m}{n}
\usepackage{url}
\usepackage{macros}

\title{Extending Consequence-Based Reasoning to $\boldsymbol{\SRIQ}$}
\author{Andrew Bate, Boris Motik, Bernardo Cuenca Grau, Franti{\v{s}}ek Siman{\v{c}\'{i}}k, Ian Horrocks\\
    Department of Computer Science, University of Oxford,\\
    Oxford, United Kingdom\\
    firstname.lastname@cs.ox.ac.uk
}
\nocopyright

\begin{document}

\maketitle

\begin{abstract}
Consequence-based calculi are a family of reasoning algorithms for
\emph{description logics} (DLs), and they combine hypertableau and resolution
in a way that often achieves excellent performance in practice. Up to now,
however, they were proposed for either Horn DLs (which do not support
disjunction), or for DLs without counting quantifiers. In this paper we present
a novel consequence-based calculus for \SRIQ---a rich DL that supports both
features. This extension is non-trivial since the intermediate consequences
that need to be derived during reasoning cannot be captured using DLs
themselves. The results of our preliminary performance evaluation suggest the
feasibility of our approach in practice.

\end{abstract}

\section{Introduction}\label{sec:introduction}

\emph{Description logics} (DLs) \cite{dl-handbook} are a family of knowledge
representation formalisms with numerous applications in practice. DL-based
applications model a domain of interest by means of an \emph{ontology}, in
which key notions in the domain are described using \emph{concepts} (i.e.,
unary predicates), and the relationships between concepts are described using
\emph{roles} (i.e., binary predicates). \emph{Subsumption} is the problem of
determining whether each instance of a concept $C$ is also an instance of a
concept $D$ in all models of an ontology, and it is a fundamental reasoning
problem in applications of DLs. For expressive DLs, this problem is of high
worst-case complexity, ranging from \ExpTime up to \NtwoExpTime.

Despite these discouraging complexity bounds, highly optimised reasoners such
as FaCT++~\cite{fact++}, Pellet~\cite{DBLP:journals/ws/SirinPGKK07},
HermiT~\cite{ghmsw14HermiT}, and Konclude
\cite{DBLP:journals/ws/SteigmillerLG14} have proved successful in practice.
These systems are typically based on (hyper)tableau calculi, which construct a
finite representation of a canonical model of the ontology disproving a
postulated subsumption. While such calculi can handle many ontologies, in some
cases they construct very large model representations, which is a source of
performance problems; this is further exacerbated by the large number of
subsumption tests often required to classify an ontology.

A recent breakthrough in DL reasoning came in the form of
\emph{consequence-based calculi}. The reasoning algorithm by \citeA{babl05} for
the lightweight logic \EL can be seen as the first such calculus. It was later
extended to the more expressive DLs Horn-\SHIQ~\cite{DBLP:conf/ijcai/Kazakov09}
and Horn-\SROIQ~\cite{DBLP:conf/kr/OrtizRS10}---DLs that support counting
quantifiers, but not disjunctions between concepts. Consequence-based calculi
were also developed for \ALCH~\cite{DBLP:conf/ijcai/SimancikKH11} and
\ALCI~\cite{smh14consequence-FPT}, which support concept disjunction, but not
counting quantifiers. Such calculi can be seen as combining resolution and
hypertableau (see \cref{sec:motivation} for details): as in resolution, they
describe ontology models by systematically deriving relevant consequences; and
as in (hyper)tableau, they are goal-directed and avoid drawing unnecessary
consequences. Additionally, they are not only refutationally complete, but can
also (dis)prove all relevant subsumptions in a single run, which can greatly
reduce the overall computational work. Finally, unlike implemented
(hyper)tableau reasoners, they are worst-case optimal for the logic they
support. \citeA{DBLP:conf/cade/SteigmillerGL14} presented a way of combining a
consequence-based calculus with a traditional tableau-based prover; while such
a combination seems to perform well in practice, the saturation rules are only
known to be complete for \EL ontologies, and the overall approach is not
worst-case optimal for \SRIQ.

Existing consequence-based algorithms cannot handle DLs such as \ALCHIQ that
provide both disjunctions and counting quantifiers. As we argue in
\cref{sec:motivation}, extending these algorithms to handle such DLs is
challenging: counting quantifiers require equality reasoning which, together
with disjunctions, can impose complex constraints on ontology models; and,
unlike existing consequence-based calculi, such constraints cannot be captured
using DLs themselves, which makes the reasoning process much more involved.
   
In \cref{sec:algorithm} we present a consequence-based calculus for \ALCHIQ; by
using the encoding of role chains by \citeA{Kazakov:08:SROIQ:KR}, our calculus
can also handle \SRIQ, which covers all of OWL 2 DL except for nominals,
reflexive roles, and datatypes. Borrowing ideas from resolution theorem
proving, we encode the calculus' consequences as first-order clauses of a
specific form, and we handle equality using a variant of \emph{ordered
paramodulation} \cite{nieuwenhuis95theorem}---a state of the art calculus for
equational theorem proving used in modern theorem provers such as
E~\cite{Schulz:AICOM-2002} and Vampire~\cite{DBLP:journals/aicom/RiazanovV02}.
Furthermore, we have carefully constrained the inference rules so that our
calculus mimics existing calculi on \ELH ontologies, which ensures robust
performance of our calculus on `mostly-\ELH' ontologies.
  
We have implemented a prototype system and compared its performance with that
of well-established reasoners. Our results in Section~\ref{sec:evaluation}
suggest that our system can significantly outperform FaCT++, Pellet, or HermiT,
and often exhibits comparable performance to that of Konclude.

\begin{table*}[tb]
\begin{center}
\caption{Translating Normalised \ALCHIQ Ontologies into DL-Clauses}\label{tab:alchiq-transformation}
\newcommand{\goesto}[1]{\multirow{#1}{*}{$\rightsquigarrow$}}
\begin{tabular}{c@{\qquad}r@{\;}l@{\qquad}c@{\qquad}r@{\;}ll}
    \hline
    \\[-1.5ex]
    DL1                     & $\bigsqcap\limits_{1 \leq i \leq n} B_i \sqsubseteq$  & $\bigsqcup\limits_{n+1 \leq i \leq m} B_i$    & \goesto{1}    & $\bigwedge\limits_{1 \leq i \leq n} B_i(x) \rightarrow$                           & $\bigvee\limits_{n+1 \leq i \leq m} B_i(x)$ \\[3ex]
    \multirow{3}{*}{DL2}    & \multirow{3}{*}{$B_1 \sqsubseteq$}                    & \multirow{3}{*}{$\atleastq{n}{S}{B_2}$}       & \goesto{3}    & $B_1(x) \rightarrow$                                                              & $S(x,f_i(x))$             & for ${1 \leq i \leq n}$ \\
                            &                                                       &                                               &               & $B_1(x) \rightarrow$                                                              & $B_2(f_i(x))$             & for ${1 \leq i \leq n}$ \\
                            &                                                       &                                               &               & $B_1(x) \rightarrow$                                                              & $f_i(x) \nequals f_j(x)$  & for ${1 \leq i < j \leq n}$ \\[1.5ex]
    DL3                     & $\exists S.B_1 \sqsubseteq$                           & $B_2$                                         & \goesto{1}    & $S(z_1,x) \wedge B_1(x) \rightarrow$                                              & $B_2(z_1)$ \\[1.5ex]
    \multirow{2}{*}{DL4}    & \multirow{2}{*}{$B_1 \sqsubseteq$}                    & \multirow{2}{*}{$\atmostq{n}{S}{B_2}$}        & \goesto{2}    & $S(z_1,x) \wedge B_2(x) \rightarrow$                                              & $S_{B_2}(z_1,x)$          & for fresh $S_{B_2}$ \\
                            &                                                       &                                               &               & $B_1(x) \wedge \bigwedge\limits_{1 \leq i \leq n+1} S_{B_2}(x,z_i) \rightarrow$   & \multicolumn{2}{@{}l@{}}{$\bigvee\limits_{1 \leq i < j \leq n+1} z_i \equals z_j$} \\[3.5ex]
    DL5                     & $S_1 \sqsubseteq$                                     & $S_2$                                         & \goesto{1}    & $S_1(z_1,x) \rightarrow$                                                          & $S_2(z_1,x)$ \\[2ex]
    DL6                     & $S_1 \sqsubseteq$                                     & $S_2^-$                                       & \goesto{1}    & $S_1(z_1,x) \rightarrow$                                                          & $S_2(x,z_1)$ \\[0.5ex]
    \hline
\end{tabular}
\end{center}
\end{table*}      

\section{Preliminaries}\label{sec:preliminaries}

\textbf{First-Order Logic.\ } It is usual in equational theorem proving to
encode atomic formulas as terms, and to use a multi-sorted signature that
prevents us from considering malformed terms. Thus, we partition the signature
into a set $\P$ of \emph{predicate symbols} and a set $\F$ of \emph{function
symbols}; moreover, we assume that $\P$ has a special constant $\TOP$. A
\emph{term} is constructed as usual using variables and the signature symbols,
with the restriction that predicate symbols are allowed to occur only at the
outermost level; the latter terms are called $\P$-\emph{terms}, while all other
terms are $\F$-\emph{terms}. For example, for $P$ a predicate and $f$ a
function symbol, $f(P(x))$ and $P(P(x))$ are both malformed; $P(f(x))$ is a
well-formed $\P$-term; and $f(x)$ and $x$ are both well-formed $\F$-terms. Term
$f(t)$ is an $f$-\emph{successor} of $t$, and $t$ is an $f$-\emph{predecessor}
of $f(t)$.

An \emph{equality} is a formula of the form ${s \equals t}$, where $s$ and $t$
are either both $\F$- or both $\P$-terms. An equality of the form ${P(\vec{s})
\equals \TOP}$ is called an \emph{atom} and is written as just $P(\vec{s})$
whenever it is clear from the context that the expression denotes a formula,
and not a $\P$-term. An \emph{inequality} is a negation of an equality and is
written as ${s \nequals t}$. We assume that $\equals$ and $\nequals$ are
implicitly symmetric---that is, ${s \eqorineq t}$ and ${t \eqorineq s}$ are
identical, for ${{\eqorineq} \in \setsingle{ {\equals}, \; {\nequals}}}$. A
\emph{literal} is an equality or an inequality. A \emph{clause} is a formula of
the form ${\forall \vec{x}.[\Gamma \rightarrow \Delta]}$ where $\Gamma$ is a
conjunction of atoms called the \emph{body}, $\Delta$ is a disjunction of
literals called the \emph{head}, and $\vec{x}$ contains all variables occurring
in the clause; quantifier $\forall \vec{x}$ is usually omitted as it is
understood implicitly. We often treat conjunctions and disjunctions as sets
(i.e., they are unordered and without repetition) and use them in standard set
operations; and we write the empty conjunction (disjunction) as $\top$
($\bot$). For $\alpha$ a term, literal, clause, or a set thereof, we say that
$\alpha$ is \emph{ground} if it does not contain a variable; $\alpha\sigma$ is
the result of applying a substitution $\sigma$ to $\alpha$; and we often write
substitutions as ${\sigma = \subst{ x \mapsto t_1, \; y \mapsto t_2, \; \dots
}}$. We use the standard notion of subterm positions; $\atPos{s}{p}$ is the
subterm of $s$ at position $p$; position $p$ is \emph{proper} in a term $t$ if
${\atPos{t}{p} \neq t}$; and $\replPos{s}{t}{p}$ is the term obtained by
replacing the subterm of $s$ at position $p$ with $t$.

A \emph{Herbrand equality interpretation} is a set of ground equalities
satisfying the usual congruence properties. Satisfaction of a ground
conjunction, a ground disjunction, or a (not necessarily ground) clause
$\alpha$ in an interpretation $I$, written ${I \models \alpha}$, as well as
entailment of a clause ${\Gamma \rightarrow \Delta}$ from a set of clauses
$\Onto$, written ${\Onto \models \Gamma \rightarrow \Delta}$, are defined as
usual. Note that a ground disjunction of literals $\Delta$ may contain
inequalities so ${I \models \Delta}$ does not necessarily imply ${I \cap \Delta
\neq \emptyset}$.

Unless otherwise stated, (possibly indexed) letters $x$, $y$, and $z$ denote
variables; $l$, $r$, $s$, and $t$ denote terms; $A$ denotes an atom or a
$\P$-term (depending on the context); $L$ denotes a literal; $f$ and $g$ denote
function symbols; $B$ denotes a unary predicate symbol; and $S$ denotes a
binary predicate symbol.

\smallskip
\noindent
\textbf{Orders.\ } A \emph{strict order} $\succ$ on a universe $U$ is an
irreflexive, asymmetric, and transitive relation on $U$; and $\succeq$ is the
\emph{non-strict order} induced by $\succ$. Order $\succ$ is \emph{total} if,
for all ${a,b \in U}$, we have ${a \succ b}$, ${b \succ a}$, or ${a = b}$.
Given ${{\circ} \in \setsingle{{\succ}, {\succeq}}}$, element ${b \in U}$, and
subset ${S \subseteq U}$, the notation ${S \circ b}$ abbreviates ${\exists a
\in S: a \circ b}$. The \emph{multiset extension} $\mul{\succ}$ of $\succ$
compares multisets $M$ and $N$ on $U$ such that ${M \mul{\succ} N}$ \iff ${M
\neq N}$ and, for each ${n \in N \setminus M}$, some ${m \in M \setminus N}$
exists such that ${m \succ n}$, where $\setminus$ is the multiset difference
operator.

A \emph{term order} $\succ$ is a \emph{strict} order on the set of all terms.
We extend $\succ$ to literals by identifying each ${s \nequals t}$ with the
multiset ${\{ s, s, t, t \}}$ and each ${s \equals t}$ with the multiset ${\{
s, t \}}$, and by comparing the result using the multiset extension of $\succ$.
We reuse the symbol $\succ$ for the induced literal order since the intended
meaning should be clear from the context.

\smallskip
\noindent
\textbf{DL-Clauses.\ } Our calculus takes as input a set $\Onto$ of
\emph{DL-clauses}---that is, clauses restricted to the following form. Let
$\P_1$ and $\P_2$ be countable sets of unary and binary predicate symbols, and
let $\F$ be a countable set of unary function symbols. DL-clauses are written
using the \emph{central variable} $x$ and variables $z_i$. A
\emph{DL-$\F$-term} has the form $x$, $z_i$, or $f(x)$ with ${f \in \F}$; a
\emph{DL-$\P$-term} has the form $B(z_i)$, $B(x)$, $B(f(x))$, $S(x,z_i)$,
$S(z_i,x)$, $S(x,f(x))$, $S(f(x),x)$ with ${B \in \P_1}$ and ${S \in \P_2}$;
and a \emph{DL-term} is a DL-$\F$-term or a DL-$\P$-term. A \emph{DL-atom} has
the form ${A \equals \TOP}$ with $A$ a DL-$\P$-term. A \emph{DL-literal} is a
DL-atom, or it is of the form ${f(x) \eqorineq g(x)}$, ${f(x) \eqorineq z_i}$,
or ${z_i \eqorineq z_j}$ with ${{\eqorineq} \in \setsingle{{\equals}, \;
{\nequals}}}$. A \emph{DL-clause} contains only DL-atoms of the form $B(x)$,
$S(x,z_i)$, and $S(z_i,x)$ in the body and only DL-literals in the head, and
each variable $z_i$ occurring in the head also occurs in the body. An
\emph{ontology} $\Onto$ is a finite set of DL-clauses. A \emph{query clause} is
a DL-clause in which all literals are of the form $B(x)$. Given an ontology
$\Onto$ and a query clause ${\Gamma \rightarrow \Delta}$, our calculus decides
whether ${\Onto \models \Gamma \rightarrow \Delta}$ holds.

\SRIQ ontologies written using the DL-style syntax can be transformed into
DL-clauses without affecting query clause entailment. First, we
\emph{normalise} DL axioms to the form shown on the left-hand side of
\cref{tab:alchiq-transformation}: we transform away role chains and then
replace all complex concepts with fresh atomic ones; this process is well
understood
\cite{DBLP:conf/ijcai/Kazakov09,Kazakov:08:SROIQ:KR,smh14consequence-FPT}, so
we omit the details. Second, using the well-known correspondence between DLs
and first-order logic \cite{dl-handbook}, we translate normalised axioms to
DL-clauses as shown on the right-hand side of \cref{tab:alchiq-transformation}.
The standard translation of ${B_1 \sqsubseteq \atmostq{n}{S}{B_2}}$ requires
atoms $B_2(z_i)$ in clause bodies, which are not allowed in our setting. We
address this issue by introducing a fresh role $S_{B_2}$ that we axiomatise as
${S(y,x) \wedge B_2(x) \rightarrow S_{B_2}(y,x)}$; this, in turn, allows us to
clausify the original axiom as if it were ${B_1 \sqsubseteq
\atmost{n}{S_{B_2}}}$. For an \ELH ontology, $\Onto$ contains DL-clauses of
type DL1 with ${m = n + 1}$, DL2 with ${n = 1}$, DL3, and DL5.

\section{Motivation}\label{sec:motivation}

\begin{figure*}[tb]
\nextequationlabel{ex:why:1}
\nextequationlabel{ex:why:2}
\nextequationlabel{ex:why:3}
\nextequationlabel{ex:why:4}

\infer{ex:why:v0-1}{\top}{B_0(x)}{\Init}
\infer{ex:why:v0-2}{\top}{S_j(x,f_{1,j}(x))}{\Hyper{ex:why:1,ex:why:v0-1}}
\infer{ex:why:v0-3}{\top}{B_1(f_{1,j}(x))}{\Hyper{ex:why:2,ex:why:v0-1}}

\nextequationlabel{eq:why:succ-v0-f11-v1}
\nextequationlabel{eq:why:succ-v0-f12-v1}

\infer{ex:why:v1-1}{\top}{S_j(y,x)}{\Succ{ex:why:v0-2,ex:why:v0-3}}
\infer{ex:why:v1-2}{\top}{B_1(x)}{\Succ{ex:why:v0-2,ex:why:v0-3}}
\infer{ex:why:v1-3}{\top}{S_j(x,f_{2,j}(x))}{\Hyper{ex:why:1,ex:why:v1-2}}
\infer{ex:why:v1-4}{\top}{B_2(x,f_{2,j}(x))}{\Hyper{ex:why:2,ex:why:v1-2}}

\infer{ex:why:vn-1}{\top}{S_j(y,x)}{\ruleword{Succ}[$\dots$]}
\infer{ex:why:vn-2}{\top}{B_n(x)}{\ruleword{Succ}[$\dots$]}
\infer{ex:why:vn-3}{\top}{C_n(x)}{\Hyper{ex:why:3,ex:why:vn-2}}
\infer{ex:why:vn-4}{\top}{C_{n-1}(y)}{\Hyper{ex:why:4,ex:why:vn-1,ex:why:vn-3}}

\infer{ex:why:v1-5}{\top}{C_1(x)}{\ruleword{Pred}[$\dots$]}
\infer{ex:why:v1-6}{\top}{C_0(y)}{\Hyper{ex:why:4,ex:why:v1-1,ex:why:v1-5}}

\infer{ex:why:v0-4}{\top}{C_0(x)}{\Pred{ex:why:v1-6}}

\begin{displaymath}
\begin{array}{r@{\;}lcr@{\;}l@{\qquad}l@{\;}l}
    \multicolumn{7}{@{}c@{}}{\text{Ontology } \Onto_1} \\
    \hline
    \multirow{2}{*}{$B_i$}  & \multirow{2}{*}{$\sqsubseteq \exists S_j.B_{i+1}$}    & \multirow{2}{*}{$\rightsquigarrow$}   & B_i(x) \rightarrow                        & S_j(x,f_{i+1,j}(x))   & \eqref{ex:why:1}  & \rdelim\}{2}{4.5cm}[$\text{ for } 0 \leq i < n \text{ and } 1 \leq j \leq 2$] \\
                            &                                                       &                                       & B_i(x) \rightarrow                        & B_{i+1}(f_{i+1,j}(x)) & \eqref{ex:why:2}  & \\
    B_n                     & \sqsubseteq C_n                                       & \rightsquigarrow                      & B_n(x) \rightarrow                        & C_n(x)                & \eqref{ex:why:3}  & \\
    \exists S_j.C_{i+1}     & \sqsubseteq C_i                                       & \rightsquigarrow                      & S_j(z_1,x) \wedge C_{i+1}(x) \rightarrow  & C_i(z_1)              & \eqref{ex:why:4}  & \;\;\;\,\text{for } 0 \leq i < n \text{ and } 1 \leq j \leq 2 \\
    \hline
\end{array}
\end{displaymath}
\vspace{0.2cm}
\begin{center}
\begin{tikzpicture}[scale=0.7, every node/.style={scale=0.7}]
    \context{v0}{0}{0}{v_{B_0(x)}}{B_0(x)}{
        \pinf{ex:why:v0-1}
        \pinf{ex:why:v0-2}
        \pinf{ex:why:v0-3}
        \pinf{ex:why:v0-4}
    } ;

    \context{v1}{8}{0}{v_{B_1(x)}}{B_1(x)}{
        \pinf{ex:why:v1-1}
        \pinf{ex:why:v1-2}
        \pinf{ex:why:v1-3}
        \pinf{ex:why:v1-4}
        \pinf{ex:why:v1-5}
        \pinf{ex:why:v1-6}
    } ;

    \coordinate (v2) at (11.5,0) ;

    \coordinate (vn-1) at (12.5,0) ;

    \node at (12,0) {$\cdots$} ;

    \context{vn}{15}{0}{v_{B_n(x)}}{B_n(x)}{
        \pinf{ex:why:vn-1}
        \pinf{ex:why:vn-2}
        \pinf{ex:why:vn-3}
        \pinf{ex:why:vn-4}
    } ;

    \draw[link] (v0) to[bend left=10]  node[above] {\edgelabel{ex:why:v0-2,ex:why:v0-3}{f_{1,1}}{eq:why:succ-v0-f11-v1}} (v1) ;
    \draw[link] (v0) to[bend left=-10] node[above] {\edgelabel{ex:why:v0-2,ex:why:v0-3}{f_{1,2}}{eq:why:succ-v0-f12-v1}} (v1) ;

    \draw[linkdots] (v1) to[bend left=15]  (v2) ;
    \draw[linkdots] (v1) to[bend left=-15] (v2) ;

    \draw[linkdots] (vn-1) to[bend left=15]  (vn) ;
    \draw[linkdots] (vn-1) to[bend left=-15] (vn) ;

\end{tikzpicture}
\end{center}\vspace{-0.5cm}
\caption{Example Motivating Consequence-Based Calculi}\label{fig:why-consequence-based}
\end{figure*}
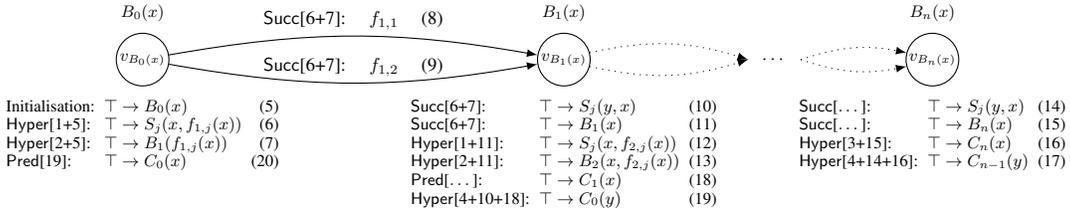

As motivation for our work, in \cref{sec:motivation:why} we discuss the
drawbacks of existing DL reasoning calculi, and then in
\cref{sec:motivation:notions} we discuss how existing consequence-based calculi
address these problems by separating clauses into contexts in a way that
considerably reduces the number of inferences. Next, in
\cref{sec:motivation:extending} we discuss the main contribution of this paper,
which lies in extending the consequence-based framework to a DL with
disjunctions and number restrictions. Handling the latter requires equality
reasoning, which requires a more involved calculus and completeness proof.

\subsection{Why Consequence-Based Calculi?}\label{sec:motivation:why}

Consider the \EL ontology $\Onto_1$ in \cref{fig:why-consequence-based}; one
can readily check that ${\Onto \models B_i(x) \rightarrow C_i(x)}$ holds for
${0 \leq i \leq n}$. To prove ${\Onto \models B_0(x) \rightarrow C_0(x)}$ using
the (hyper)tableau calculus, we start with $B_0(a)$ and apply
\eqref{ex:why:1}--\eqref{ex:why:4} in a forward-chaining manner. Since $\Onto$
contains \eqref{ex:why:1} for ${j \in \setsingle{1, 2}}$, this constructs a
tree-shaped model of depth $n$ and a fanout of two, where nodes at depth $i$
are labelled by $B_i$ and $C_i$. Forward chaining ensures that reasoning is
goal-oriented; however, all nodes labelled with $B_i$ are of the same type and
they share the same properties, which reveals a weakness of (hyper)tableau
calculi: the constructed models can be large (exponential in our example) and
highly redundant; apart from causing problems in practice, this often prevents
(hyper)tableau calculi from being worst-case optimal. Techniques such as
\emph{caching} \cite{DBLP:conf/tableaux/GoreN07} or \emph{anywhere blocking}
\cite{msh09hypertableau} can constrain model construction, but their
effectiveness often depends on the order of rule applications. Thus, model size
is a key limiting factor for (hyper)tableau-based reasoners
\cite{msh09hypertableau}.

In contrast, resolution describes models using (universally quantified) clauses
that `summarise' the model. This eliminates redundancy and ensures worst-case
optimality of many resolution decision procedures. Many resolution variants
have been proposed \cite{BachmairGanzinger:HandbookAR:resolution:2001}, each
restricting inferences in a specific way. However, to ensure termination, all
decision procedure for DLs we are aware of perform inferences with the
`deepest' and the `covering' clause atoms, so all of them will resolve all
\eqref{ex:why:1} with all \eqref{ex:why:4} to obtain all $2 n^2$ clauses of the
form
\begin{align}
    \begin{array}{@{}l@{}}
        B_i(x) \wedge C_{k+1}(f_{i+1,j}(x)) \rightarrow C_k(x) \\
        \hspace{2.5cm} \text{for } 1 \leq i,k < n \text{ and } 1 \leq j \leq 2.
    \end{array}
    \label{ex:why:conclusions}
\end{align}
Of these $2 n^2$ clauses, only those with ${i = k}$ are relevant to proving our
goal. If we extend $\Onto$ with additional clauses that contain $B_i$ and
$C_i$, each of these $2 n^2$ clauses can participate in further inferences and
give rise to more irrelevant clauses. This problem is particularly pronounced
when $\Onto$ is satisfiable since we must then produce all consequences of
$\Onto$.

\subsection{Basic Notions}\label{sec:motivation:notions}

Consequence-based calculi combine `summarisation' of resolution with
goal-directed search of (hyper)tableau calculi. \citeA{smh14consequence-FPT}
presented a framework for \ALCI capturing the key elements of the related
calculi by \citeA{babl05}, \citeA{DBLP:conf/ijcai/Kazakov09},
\citeA{DBLP:conf/kr/OrtizRS10}, and \citeA{DBLP:conf/ijcai/SimancikKH11}.
Before extending this framework to \ALCHIQ in \cref{sec:algorithm}, we next
informally recapitulate the basic notions; however, to make this paper easier
to follow, we use the same notation and terminology as in \cref{sec:algorithm}.

Our consequence-based calculus constructs a directed graph ${\D =
\tuple{\V,\E,\Ssym,\coresym,\csuccsym}}$ called a \emph{context structure}. The
vertices in $\V$ are called \emph{contexts}. Let $I$ be a Herbrand model of
$\Onto$; hence, the domain of $I$ contains ground terms. Instead of
representing each ground term of $I$ separately as in (hyper)tableau calculi,
$\D$ can represent the properties of several terms by a single context $v$.
Each context ${v \in \V}$ is associated with a (possibly empty) conjunction
$\core{v}$ of \emph{core} atoms that must hold for all ground terms that $v$
represents; thus, $\core{v}$ determines the `kind' of context $v$. Moreover,
$v$ is associated with a set $\S{v}$ of clauses that capture the constraints
that these terms must satisfy. Partitioning clauses into sets allows us to
restrict the inferences between clause sets and thus eliminate certain
irrelevant inferences. Clauses in $\S{v}$ are `relative' to $\core{v}$: for
each ${\Gamma \rightarrow \Delta \in \S{v}}$, we have ${\Onto \models \core{v}
\wedge \Gamma \rightarrow \Delta}$---that is, we do not include $\core{v}$ in
clause bodies since $\core{v}$ holds implicitly. Function $\csuccsym$ provides
each context ${v \in \V}$ with a concept order $\csucc{v}$ that restricts
resolution inferences in the presence of disjunctions.

Contexts are connected by directed edges labelled with function symbols. If $u$
is connected to $v$ via an $f$-labelled edge, then the $f$-successor of each
ground term represented by $u$ is represented by $v$. Conversely, if $u$ and
$v$ are \emph{not} connected by an $f$-edge, then each ground term represented
by $v$ is not an $f$-successor of a ground term represented by $u$, so no
inference between $\S{u}$ and $\S{v}$ is ever needed.

Consequence-based calculi are not just complete for refutation: they derive the
required consequences. \Cref{fig:why-consequence-based} demonstrates this for
${\Onto_1 \models B_0(x) \rightarrow C_0(x)}$. The cores and the clauses shown
above and below, respectively, each context, and clause numbers correspond to
the derivation order. To prove ${B_0(x) \rightarrow C_0(x)}$, we introduce
context $v_{B_0(x)}$ with core ${B_0(x)}$ and add clause \eqref{ex:why:v0-1} to
it. The latter says that $B_0$ holds for $a$, and it is analogous to
initialising a (hyper)tableau calculus with $B_0(a)$. The calculus then applies
rules from \cref{table:rules} to derive new clauses and/or extend $\D$.

\ruleword{Hyper} is the standard hyperresolution rule restricted to a single
context at a time. Thus, we derive \eqref{ex:why:v0-2} from \eqref{ex:why:1}
and \eqref{ex:why:v0-1}, and \eqref{ex:why:v0-3} from \eqref{ex:why:2} and
\eqref{ex:why:v0-1}. Hyperresolution resolves all body atoms, which makes the
resolvent relevant for the context and prevents the derivation of irrelevant
clauses such as \eqref{ex:why:conclusions}.

Context $v_{B_0(x)}$ contains atoms with function symbols $f_{1,1}$ and
$f_{1,2}$, so the \ruleword{Succ} rule must ensure that the $f_{1,1}$- and
$f_{1,2}$-successors of the ground terms represented by $v_{B_0(x)}$ are
adequately represented in $\D$. We can control context introduction via a
parameter called an \emph{expansion strategy}---a function that determines
whether to reuse an existing context or introduce a fresh one; in the latter
case, it also determines how to initialise the context's core. We discuss
possible strategies in \cref{sec:algorithm:definitions}; in the rest of this
example, we use the so-called cautious strategy, where the \ruleword{Succ} rule
introduces context $v_{B_1(x)}$ and initialises it with \eqref{ex:why:v1-1} and
\eqref{ex:why:v1-2}. Note that \eqref{ex:why:v0-2} represents two clauses, both
of which we satisfy (in separate applications of the \ruleword{Succ} rule)
using $v_{B_1(x)}$.

We construct contexts ${v_{B_2(x)}, \dots, v_{B_n(x)}}$ analogously, we derive
\eqref{ex:why:vn-3} by hyperresolving \eqref{ex:why:3} and \eqref{ex:why:vn-1},
and we derive \eqref{ex:why:vn-4} by hyperresolving \eqref{ex:why:4},
\eqref{ex:why:vn-1}, and \eqref{ex:why:vn-3}. Clause \eqref{ex:why:vn-4}
imposes a constraint on the predecessor context, which we propagate using the
\ruleword{Pred} rule, deriving \eqref{ex:why:v1-6} and \eqref{ex:why:v0-4}.
Since clauses of $v_{B_0(x)}$ are `relative' to the core of $v_{B_0(x)}$,
clause \eqref{ex:why:v0-4} represents our query clause, as required.

\subsection{Extending the Framework to $\boldsymbol{\ALCHIQ}$}\label{sec:motivation:extending}

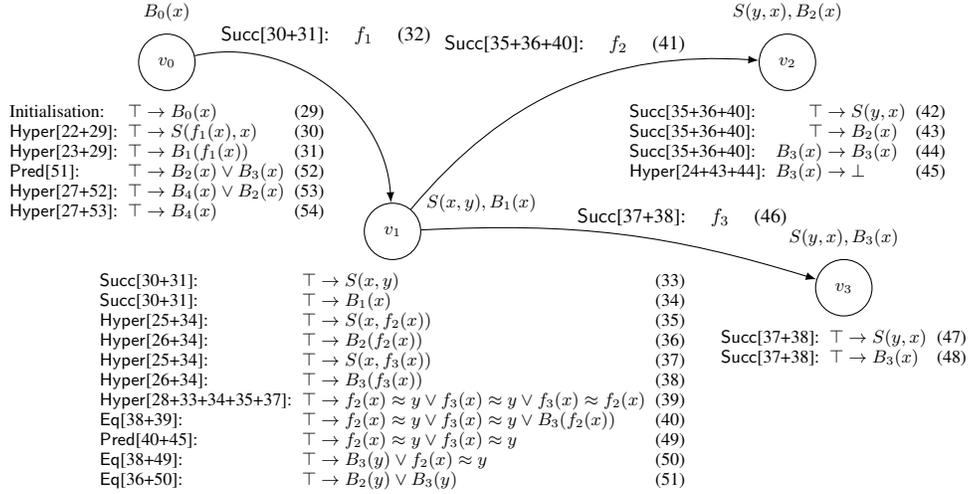
\begin{figure*}[tb]
\nextequationlabel{ex:diff:1}
\nextequationlabel{ex:diff:2}
\nextequationlabel{ex:diff:3}
\nextequationlabel{ex:diff:4}
\nextequationlabel{ex:diff:5}
\nextequationlabel{ex:diff:6}
\nextequationlabel{ex:diff:7}

\begin{displaymath}
\begin{array}{r@{\;}lcr@{\;}ll@{\;}l}
    \multicolumn{7}{@{}c@{}}{\text{Ontology } \Onto_2} \\
    \hline
    \multirow{2}{*}{$B_0$}  & \multirow{2}{*}{$\sqsubseteq \exists S^-.B_1$}    & \multirow{2}{*}{$\rightsquigarrow$}   & B_0(x)                & \rightarrow S(f_1(x),x)                                                                                           & \eqref{ex:diff:1} \\
                            &                                                   &                                       & B_0(x)                & \rightarrow B_1(f_1(x))                                                                                           & \eqref{ex:diff:2} \\
    \multirow{2}{*}{$B_1$}  & \multirow{2}{*}{$\sqsubseteq \exists S.B_i$}      & \multirow{2}{*}{$\rightsquigarrow$}   & B_1(x)                & \rightarrow S(x,f_i(x))                                                                                           & \eqref{ex:diff:4} & \rdelim\}{3}{2.35cm}[$\text{ for } 2 \leq i \leq 3$] \\
                            &                                                   &                                       & B_1(x)                & \rightarrow B_i(f_i(x))                                                                                           & \eqref{ex:diff:5} \\
    B_i                     & \sqsubseteq B_4                                   & \rightsquigarrow                      & B_i(x)                & \rightarrow B_4(x)                                                                                                & \eqref{ex:diff:6} \\
    B_2 \sqcap B_3          & \sqsubseteq \bot                                  & \rightsquigarrow                      & B_2(x) \wedge B_3(x)  & \rightarrow \bot                                                                                                  & \eqref{ex:diff:3} \\
    B_1                     & \sqsubseteq {\leq} 2.S                            & \rightsquigarrow                      & \multicolumn{2}{@{}l@{}}{B_1(x) \wedge \bigwedge_{1 \leq i \leq 3} S(x,z_i) \rightarrow \bigvee_{1 \leq j < k \leq 3} z_j \equals z_k}    & \eqref{ex:diff:7} \\
    \hline
\end{array}
\end{displaymath}

\infer{ex:diff:v1-1}{\top}{B_0(x)}{\Init}
\infer{ex:diff:v1-2}{\top}{S(f_1(x),x)}{\Hyper{ex:diff:1,ex:diff:v1-1}}
\infer{ex:diff:v1-3}{\top}{B_1(f_1(x))}{\Hyper{ex:diff:2,ex:diff:v1-1}}

\nextequationlabel{ex:diff:succ-v1-v2}

\infer{ex:diff:v2-1}{\top}{S(x,y)}{\Succ{ex:diff:v1-2,ex:diff:v1-3}}
\infer{ex:diff:v2-2}{\top}{B_1(x)}{\Succ{ex:diff:v1-2,ex:diff:v1-3}}
\infer{ex:diff:v2-3}{\top}{S(x,f_2(x))}{\Hyper{ex:diff:4,ex:diff:v2-2}}
\infer{ex:diff:v2-4}{\top}{B_2(f_2(x))}{\Hyper{ex:diff:5,ex:diff:v2-2}}
\infer{ex:diff:v2-5}{\top}{S(x,f_3(x))}{\Hyper{ex:diff:4,ex:diff:v2-2}}
\infer{ex:diff:v2-6}{\top}{B_3(f_3(x))}{\Hyper{ex:diff:5,ex:diff:v2-2}}
\infer{ex:diff:v2-7}{\top}{f_2(x) \equals y \vee f_3(x) \equals y \vee f_3(x) \equals f_2(x)}{\Hyper{ex:diff:7,ex:diff:v2-1,ex:diff:v2-2,ex:diff:v2-3,ex:diff:v2-5}}
\infer{ex:diff:v2-8}{\top}{f_2(x) \equals y \vee f_3(x) \equals y \vee B_3(f_2(x))}{\Eq{ex:diff:v2-6,ex:diff:v2-7}}

\nextequationlabel{ex:diff:succ-v2-v3}

\infer{ex:diff:v3-1}{\top}{S(y,x)}{\Succ{ex:diff:v2-3,ex:diff:v2-4,ex:diff:v2-8}}
\infer{ex:diff:v3-2}{\top}{B_2(x)}{\Succ{ex:diff:v2-3,ex:diff:v2-4,ex:diff:v2-8}}
\infer{ex:diff:v3-3}{B_3(x)}{B_3(x)}{\Succ{ex:diff:v2-3,ex:diff:v2-4,ex:diff:v2-8}}
\infer{ex:diff:v3-4}{B_3(x)}{\bot}{\Hyper{ex:diff:3,ex:diff:v3-2,ex:diff:v3-3}}

\nextequationlabel{ex:diff:succ-v2-v4}

\infer{ex:diff:v4-1}{\top}{S(y,x)}{\Succ{ex:diff:v2-5,ex:diff:v2-6}}
\infer{ex:diff:v4-2}{\top}{B_3(x)}{\Succ{ex:diff:v2-5,ex:diff:v2-6}}

\infer{ex:diff:v2-9}{\top}{f_2(x) \equals y \vee f_3(x) \equals y}{\Pred{ex:diff:v2-8,ex:diff:v3-4}}
\infer{ex:diff:v2-10}{\top}{B_3(y) \vee f_2(x) \equals y}{\Eq{ex:diff:v2-6,ex:diff:v2-9}}
\infer{ex:diff:v2-11}{\top}{B_2(y) \vee B_3(y)}{\Eq{ex:diff:v2-4,ex:diff:v2-10}}

\infer{ex:diff:v1-4}{\top}{B_2(x) \vee B_3(x)}{\Pred{ex:diff:v2-11}}
\infer{ex:diff:v1-5}{\top}{B_4(x) \vee B_2(x)}{\Hyper{ex:diff:6,ex:diff:v1-4}}
\infer{ex:diff:v1-6}{\top}{B_4(x)}{\Hyper{ex:diff:6,ex:diff:v1-5}}

\begin{center}
\begin{tikzpicture}[scale=0.75, every node/.style={scale=0.75}]

    \context{v1}{0}{0}{v_0}{B_0(x)}{
        \pinf{ex:diff:v1-1}
        \pinf{ex:diff:v1-2}
        \pinf{ex:diff:v1-3}
        \pinf{ex:diff:v1-4}
        \pinf{ex:diff:v1-5}
        \pinf{ex:diff:v1-6}
    } ;

    \context[right=3mm]{v2}{4}{-3}{v_1}{\;\; S(x,y), B_1(x)}{
        \pinf{ex:diff:v2-1}
        \pinf{ex:diff:v2-2}
        \pinf{ex:diff:v2-3}
        \pinf{ex:diff:v2-4}
        \pinf{ex:diff:v2-5}
        \pinf{ex:diff:v2-6}
        \pinf{ex:diff:v2-7}
        \pinf{ex:diff:v2-8}
        \pinf{ex:diff:v2-9}
        \pinf{ex:diff:v2-10}
        \pinf{ex:diff:v2-11}
    } ;

    \context{v3}{11}{0}{v_2}{S(y,x), B_2(x)}{
        \pinf{ex:diff:v3-1}
        \pinf{ex:diff:v3-2}
        \pinf{ex:diff:v3-3}
        \pinf{ex:diff:v3-4}
    } ;
    
    \context{v4}{12}{-4}{v_3}{S(y,x), B_3(x)}{
        \pinf{ex:diff:v4-1}
        \pinf{ex:diff:v4-2}
    } ;

    \draw[link] (v1) to[bend left=50] node[above=6mm] {\edgelabel{ex:diff:v1-2,ex:diff:v1-3}{f_1}{ex:diff:succ-v1-v2}} (v2) ;
    \draw[link] (v2) to[bend left=27] node[above=5mm] {\edgelabel{ex:diff:v2-3,ex:diff:v2-4,ex:diff:v2-8}{f_2}{ex:diff:succ-v2-v3}} (v3) ;
    \draw[link] (v2) to[bend left=10] node[above=3mm, right=-9mm] {\edgelabel{ex:diff:v2-5,ex:diff:v2-6}{f_3}{ex:diff:succ-v2-v4}} (v4) ;

\end{tikzpicture}
\end{center}
\caption{Challenges in Extending the Consequence-Based Framework to \ALCHIQ}\label{fig:challenges}
\end{figure*}

In all consequence-based calculi presented thus far, the constraints that the
ground terms represented by a context $v$ must satisfy can be represented using
standard DL-style axioms. For example, for \ALCI, \citeA{smh14consequence-FPT}
represented all relevant consequences using DL axioms of the following form:
\begin{align}
    \bigsqcap B_i \sqsubseteq \bigsqcup B_j \sqcup \bigsqcup \exists S_k.B_k \sqcup \bigsqcup \forall S_\ell.B_\ell
\end{align}    

\ALCHIQ provides both counting quantifiers and disjunctions, the interplay of
which may impose constraints that cannot be represented in \ALCHIQ. Let
$\Onto_2$ be as in \cref{fig:challenges}. To see that ${\Onto_2 \models B_0(x)
\rightarrow B_4(x)}$ holds, we construct a Herbrand interpretation $I$ from
$B_0(a)$: \eqref{ex:diff:1} and \eqref{ex:diff:2} derive $S(f_1(a),a)$ and
$B_1(f_1(a))$; and \eqref{ex:diff:4} and \eqref{ex:diff:5} derive
$S(f_1(a),f_2(f_1(a)))$ and $B_2(f_2(f_1(a)))$, and $S(f_1(a),f_3(f_1(a)))$ and
$B_3(f_3(f_1(a)))$. Due to \eqref{ex:diff:6} we derive $B_4(f_2(f_1(a)))$ and
$B_4(f_3(f_1(a)))$. Finally, from \eqref{ex:diff:7} we derive the following
clause:
\begin{align}
\begin{array}{@{}l@{}}
    f_2(f_1(a)) \equals a \vee f_3(f_1(a)) \equals a \; \vee \\
    \hspace{4cm} f_3(f_1(a)) \equals f_2(f_1(a)) \\
\end{array} \label{ex:motivation:constraint}
\end{align}
Disjunct ${f_3(f_1(a)) \equals f_2(f_1(a))}$ cannot be satisfied due to
\eqref{ex:diff:3}; but then, regardless of whether we choose to satisfy
${f_3(f_1(a)) \equals a}$ or ${f_2(f_1(a)) \equals a}$, we derive $B_4(a)$.

Our calculus must be able to capture constraint
\eqref{ex:motivation:constraint} and its consequences, but standard DL axioms
cannot explicitly refer to specific successors and predecessors. Instead, we
capture consequences using \emph{context clauses}---clauses over terms $x$,
$f_i(x)$, and $y$, where variable $x$ represents the ground terms that a
context stands for, $f_i(x)$ represents $f_i$-successors of $x$, and $y$
represents the predecessor of $x$. We can thus identify the predecessor and the
successors of $x$ `by name', allowing us to capture constraint
\eqref{ex:motivation:constraint} as
\begin{align}
    f_2(x) \equals y \vee f_3(x) \equals y \vee f_3(x) \equals f_2(x).
\end{align}
Based on this idea, we adapted the rules by \citeA{smh14consequence-FPT} to
handle context clauses correctly, and we added rules that capture the
consequences of equality. The resulting set of rules is shown in
\cref{table:rules}.

\Cref{fig:challenges} shows how to verify ${\Onto_2 \models B_0(x) \rightarrow
B_4(x)}$ using our calculus; the maximal literal of each clause is shown on the
right. We next discuss the inferences in detail.

We first create context $v_0$ and initialise it with \eqref{ex:diff:v1-1}; this
ensures that each interpretation represented by the context structure contains
a ground term for which $B_0$ holds. Next, we derive \eqref{ex:diff:v1-2} and
\eqref{ex:diff:v1-3} using hyperresolution. At this point, we could
hyperresolve \eqref{ex:diff:4} and \eqref{ex:diff:v1-3} to obtain ${\top
\rightarrow S(f_1(x),f_2(f_1(x)))}$; however, this could easily lead to
nontermination of the calculus due to increased term nesting. Therefore, we
require hyperresolution to map variable $x$ in the DL-clauses to variable $x$
in the context clauses; thus, hyperresolution derives in each context only
consequences about $x$, which prevents redundant derivations.

The \ruleword{Succ} rule next handles function symbol $f_1$ in clauses
\eqref{ex:diff:v1-2} and \eqref{ex:diff:v1-3}. To determine which information
to propagate to a successor, \cref{def:triggers} in \cref{sec:algorithm}
introduces a set $\SUtriggers$ of \emph{successor triggers}. In our example,
DL-clause \eqref{ex:diff:7} contains atoms $B_1(x)$ and $S(x,z_i)$ in its body,
and $z_i$ can be mapped to a predecessor or a successor of $x$; thus, a context
in which hyperresolution is applied to \eqref{ex:diff:7} will be interested in
information about its predecessors, which we reflect by adding $B_1(x)$ and
$S(x,y)$ to $\SUtriggers$. In this example we use the so-called eager strategy
(see \cref{sec:algorithm:definitions}), so the \ruleword{Succ} rule introduces
context $v_1$, sets its core to $B_1(x)$ and $S(x,y)$, and initialises the
context with \eqref{ex:diff:v2-1} and \eqref{ex:diff:v2-2}.

We next introduce \eqref{ex:diff:v2-3}--\eqref{ex:diff:v2-6} using
hyperresolution, at which point we have sufficient information to apply
hyperresolution to \eqref{ex:diff:7} to derive \eqref{ex:diff:v2-7}. Please
note how the presence of \eqref{ex:diff:v2-1} is crucial for this inference.

We use paramodulation to deal with equality in clause \eqref{ex:diff:v2-7}. As
is common in resolution-based theorem proving, we order the literals in a
clause and apply inferences only to maximal literals; thus, we derive
\eqref{ex:diff:v2-8}.

Clauses \eqref{ex:diff:v2-3}, \eqref{ex:diff:v2-4}, and \eqref{ex:diff:v2-8}
contain function symbol $f_2$, so the \ruleword{Succ} rule introduces context
$v_2$. Due to clause \eqref{ex:diff:v2-4}, $B_2(x)$ holds for all ground terms
that $v_2$ represents; thus, we add $B_2(x)$ to $\core{v_2}$. In contrast, atom
$B_3(f_2(x))$ occurs in clause \eqref{ex:diff:v2-8} in a disjunction, which
means it may not hold in $v_2$; hence, we add $B_3(x)$ to the body of clause
\eqref{ex:diff:v3-3}. The latter clause allows us to derive
\eqref{ex:diff:v3-4} using hyperresolution.

Clause \eqref{ex:diff:v3-4} essentially says `$B_3(f_2(x))$ should not hold in
the predecessor', which the \ruleword{Pred} rule propagates to $v_1$ as clause
\eqref{ex:diff:v2-9}; one can understand this inference as hyperresolution of
\eqref{ex:diff:v2-8} and \eqref{ex:diff:v3-4} while observing that term
$f_2(x)$ in context $v_1$ is represented as variable $x$ in context $v_2$.

After two paramodulation steps, we derive clause \eqref{ex:diff:v2-11}, which
essentially says `the predecessor must satisfy $B_2(x)$ or $B_3(x)$'. The set
$\PRtriggers$ of \emph{predecessor triggers} from \cref{def:triggers}
identifies this as relevant to $v_0$: the DL-clauses in \eqref{ex:diff:6}
contain $B_2(x)$ and $B_3(x)$ in their bodies, which are represented in $v_1$
as $B_2(y)$ and $B_3(y)$. Hence $\PRtriggers$ contains $B_2(y)$ and $B_3(y)$,
allowing the \ruleword{Pred} rule to derive \eqref{ex:diff:v1-4}.

After two more steps, we finally derive our target clause \eqref{ex:diff:v1-6}.
We could not do this if $B_4(x)$ were maximal in \eqref{ex:diff:v1-5}; thus, we
require all atoms in the head of a goal clause to be smallest. A similar
observation applies to $\PRtriggers$: if $B_3(y)$ were maximal in
\eqref{ex:diff:v2-10}, we would not derive \eqref{ex:diff:v2-11} and propagate
it to $v_0$; thus, all atoms in $\PRtriggers$ must be smallest too.

\section{Formalising the Algorithm}\label{sec:algorithm}

In this section, we first present our consequence-based algorithm for \ALCHIQ
formally, and then we present an outline of the completeness proof; full proofs
are given in the appendix.

\subsection{Definitions}\label{sec:algorithm:definitions}

Our calculus manipulates \emph{context clauses}, which are constructed from
\emph{context terms} and \emph{context literals} as described in
\cref{def:context-terms}. Unlike in general resolution, we restrict context
clauses to contain only variables $x$ and $y$, which have a special meaning in
our setting: variable $x$ represents a ground term in a Herbrand model, and $y$
represents the predecessor of $x$; this naming convention is important for the
rules of our calculus. This is in contrast to the DL-clauses of an ontology,
which can contain variables $x$ and $z_i$, and where $z_i$ refer to either the
predecessor or a successor of $x$.

\begin{definition}\label{def:context-terms}
    A \emph{context $\F$-term} is a term of the form $x$, $y$, or $f(x)$ for
    ${f \in \F}$; a \emph{context $\P$-term} is a term of the form $B(y)$,
    $B(x)$, $B(f(x))$, $S(x,y)$, $S(y,x)$, $S(x,f(x))$, or $S(f(x),x)$ for
    ${B,R \in \P}$ and ${f \in \F}$; and a \emph{context term} is an $\F$-term
    or a $\P$-term. A \emph{context literal} is a literal of the form ${A
    \equals \TOP}$ (called a \emph{context atom}), ${f(x) \eqorineq g(x)}$, or
    ${f(x) \eqorineq y}$, ${y \eqorineq y}$, for $A$ a context $\P$-term and
    ${{\eqorineq} \in \setsingle{{\equals}, \; {\nequals}}}$. A \emph{context
    clause} is a clause with only function-free context atoms in the body, and
    only context literals in the head.
\end{definition}

\Cref{def:triggers} introduces sets $\SUtriggers$ and $\PRtriggers$, that
identify the information that must be exchanged between adjacent contexts.
Intuitively, $\SUtriggers$ contains atoms that are of interest to a context's
successor, and it guides the \ruleword{Succ} rule whereas $\PRtriggers$
contains atoms that are of interest to a context's predecessor and it guides
the \ruleword{Pred} rule.

\begin{definition}\label{def:triggers}
    The set $\SUtriggers$ of \emph{successor triggers} of an ontology $\Onto$
    is the smallest set of atoms such that, for each clause ${\Gamma
    \rightarrow \Delta \in \Onto}$,
    \begin{itemize}
         \item ${B(x) \in \Gamma}$ implies ${B(x) \in \SUtriggers}$,
         \item ${S(x,z_i) \in \Gamma}$ implies ${S(x,y) \in \SUtriggers}$, and
         \item ${S(z_i,x) \in \Gamma}$ implies ${S(y,x) \in \SUtriggers}$.
    \end{itemize}     
    The set $\PRtriggers$ of \emph{predecessor triggers} of $\Onto$ is defined
    as
    \begin{displaymath}
    \begin{array}{@{}r@{\;}l@{}}
        \PRtriggers =   & \setbuilder{A\subst{x \mapsto y, \; y \mapsto x}}{A \in \SUtriggers} \; \cup \\
                        & \setbuilder{B(y)}{B \text{ occurs in } \Onto}.
    \end{array}
    \end{displaymath}
\end{definition}

As in resolution, we restrict the inferences using a term order $\csuccsym$.
\Cref{def:context:term:order} specifies the conditions that the order must
satisfy. \Cref{cond:term:order:f-term:1,cond:term:order:f-term:2} ensure that
$\F$-terms are compared uniformly across contexts; however, $\P$-terms can be
compared in different ways in different contexts.
\Cref{cond:term:order:f-term:1,cond:term:order:f-term:2,cond:term:order:mon,cond:term:order:subterm}
ensure that, if we ground the order by mapping $x$ to a term $t$ and $y$ to the
predecessor of $t$, we obtain a \emph{simplification order}
\cite{baader98term}---a kind of term order commonly used in equational theorem
proving. Finally, \cref{cond:term:order:PR} ensures that atoms that might be
propagated to a context's predecessor via the \ruleword{Pred} rule are
smallest, which is important for completeness.

\begin{definition}\label{def:context:term:order}
    Let $\funsucc$ be a total, well-founded order on function symbols. A
    \emph{context term order} $\csuccsym$ is an order on context terms
    satisfying the following conditions:
    \begin{enumerate}
        \item\label[condition]{cond:term:order:f-term:1}
        for each ${f \in \F}$, we have ${f(x) \csuccsym x \csuccsym y}$;

        \item\label[condition]{cond:term:order:f-term:2}
        for all ${f,g \in \F}$ with ${f \funsucc g}$, we have ${f(x) \csuccsym
        g(x)}$;

        \item\label[condition]{cond:term:order:mon}
        for all terms $s_1$, $s_2$, and $t$ and each position $p$ in $t$, if
        ${s_1 \csuccsym s_2}$, then ${\replPos{t}{s_1}{p} \csuccsym
        \replPos{t}{s_2}{p}}$;

        \item\label[condition]{cond:term:order:subterm}
        for each term $s$ and each proper position $p$ in $s$, we have ${s
        \csuccsym \atPos{s}{p}}$; and

        \item\label[condition]{cond:term:order:PR}
        for each atom ${A \equals \TOP \in \PRtriggers}$ and each context term
        ${s \not\in \setsingle{x,y}}$, we have ${A \not\csuccsym s}$.
    \end{enumerate}
    Each term order is extended to a literal order, also written $\csuccsym$,
    as described in \cref{sec:preliminaries}.
\end{definition}

A lexicographic path order (LPO) \cite{baader98term} over context $\F$-terms
and context $\P$-terms, in which $x$ and $y$ are treated as constants such that
${x \csuccsym y}$, satisfies
\cref{cond:term:order:f-term:1,cond:term:order:f-term:2,cond:term:order:mon,cond:term:order:subterm}.
Furthermore, $\PRtriggers$ contains only atoms of the form $B(y)$, $S(x,y)$,
and $S(y,x)$, which we can always make smallest in the ordering; thus,
\cref{cond:term:order:PR} does not contradict the other conditions. Hence, an
LPO that is relaxed for \cref{cond:term:order:PR} satisfies
\cref{def:context:term:order}, and thus, for any given $\funsucc$, at least one
context term order exists.

Apart from orders, effective redundancy elimination techniques are critical to
efficiency of resolution calculi. \Cref{def:redundancy} defines a notion
compatible with our setting.

\begin{definition}\label{def:redundancy}
    A set of clauses $U$ \emph{contains a clause} ${\Gamma \rightarrow \Delta}$
    \emph{up to redundancy}, written ${\Gamma \rightarrow \Delta \strin U}$, if
    \begin{enumerate}
        \item\label[condition]{cond:redundancy:tautology}
        ${\{ s \equals s', \; s \nequals s' \} \subseteq \Delta}$ or ${s
        \equals s \in \Delta}$ for some terms $s$ and $s'$, or

        \item\label[condition]{cond:redundancy:strengthening}
        ${\Gamma' \subseteq \Gamma}$ and ${\Delta' \subseteq \Delta}$ for some
        clause ${\Gamma' \rightarrow \Delta' \in U}$.
    \end{enumerate}
\end{definition}

Intuitively, if $U$ contains ${\Gamma \rightarrow \Delta}$ up to redundancy,
then adding ${\Gamma \rightarrow \Delta}$ to $U$ will not modify the
constraints that $U$ represents because either ${\Gamma \rightarrow \Delta}$ is
a tautology or $U$ contains a stronger clause. Note that tautologies of the
form ${A \rightarrow A}$ are \emph{not} redundant in our setting as they are
used to initialise contexts; however, whenever our calculus derives a clause
${A \rightarrow A \vee A'}$, the set of clauses will have been initialised with
${A \rightarrow A}$, which makes the former clause redundant by
\cref{cond:redundancy:strengthening} of \cref{def:redundancy}. Moreover, clause
heads are subjected to the usual tautology elimination rules; thus, clauses
${\gamma \rightarrow \Delta \vee s \equals s}$ and ${\Gamma \rightarrow \Delta
\vee s \equals t \vee s \nequals t}$ can be eliminated.
\Cref{prop:del:redundant} shows that we can remove from $U$ each clause $C$
that is contained in ${U \setminus \setsingle{C}}$ up to redundancy; the
\ruleword{Elim} uses this to support clause subsumption.

\begin{proposition}\label{prop:del:redundant}
    For $U$ a set of clauses and $C$ and $C'$ clauses with ${C \strin U
    \setminus \setsingle{ C }}$ and ${C' \strin U}$, we have ${C' \strin U
    \setminus \setsingle{ C }}$.
\end{proposition}

We are finally ready to formalise the notion of a context structure, as well as
a notion of context structure soundness. The latter captures the fact that
context clauses from a set $\S{v}$ do not contain $\core{v}$ in their bodies.
We shall later show that our inference rules preserve context structure
soundness, which essentially proves that all clauses derived by our calculus
are indeed conclusions of the ontology in question.

\begin{definition}\label{def:contexts}
    A \emph{context structure} for an ontology $\Onto$ is a tuple ${\D =
    \tuple{\V,\E,\Ssym,\coresym,\csuccsym}}$, where $\V$ is a finite set of
    \emph{contexts}, ${\E \subseteq \V \times \V \times \F}$ is a finite set of
    edges each labelled with a function symbol, function $\coresym$ assigns to
    each context $v$ a conjunction $\core{v}$ of atoms over the $\P$-terms from
    $\SUtriggers$, function $\Ssym$ assigns to each context $v$ a finite set
    $\S{v}$ of context clauses, and function $\csuccsym$ assigns to each
    context $v$ a context term order $\csucc{v}$. A context structure $\D$ is
    \emph{sound} for $\Onto$ if the following conditions both hold.
    \begin{enumerate}[label={S\arabic*.},ref={S\arabic*},leftmargin=2.5em]
        \item\label[condition]{soundness:cond:clauses}
        For each context ${v \in \V}$ and each clause ${\Gamma \rightarrow
        \Delta \in \S{v}}$, we have ${\Onto \models \core{v} \wedge \Gamma
        \rightarrow \Delta}$.

        \item\label[condition]{soundness:cond:edges}
        For each edge ${\tuple{u,v,f} \in \E}$, we have
        \begin{displaymath}
            \Onto \models \core{u} \rightarrow \core{v} \subst{x \mapsto f(x), y \mapsto x}.
        \end{displaymath}
    \end{enumerate}
\end{definition}

\begin{table}[p]
\caption{Rules of the Consequence-Based Calculus}\label{table:rules}
\centering
\begin{tabular}{@{}l@{\;}l@{}}
\hline
\hline

\multicolumn{2}{@{}c@{}}{\textbf{\ruleword{Core} rule}} \\
If      & $A \in \core{v}$, \\
        & and $\top \rightarrow A \notin \S{v}$, \\
then    & add $\top \rightarrow A$ to $\S{v}$. \\
\hline
\hline

\multicolumn{2}{@{}c@{}}{\textbf{\ruleword{Hyper} rule}} \\
If      & $\bigwedge_{i=1}^n A_i \rightarrow \Delta \in \Onto$, \\
        & $\sigma$ is a substitution such that $\sigma(x) = x$, \\
        & $\Gamma_i \rightarrow \Delta_i \vee A_i\sigma \in \S{v}$ s.t.~$\Delta_i \not\csucceq{v} A_i\sigma$ for $1 \leq i \leq n$, \\
        & and $\bigwedge_{i=1}^n \Gamma_i \rightarrow \Delta\sigma \vee \bigvee_{i=1}^n \Delta_i \not\strin \S{v}$, \\
then    & add $\bigwedge_{i=1}^n \Gamma_i \rightarrow \Delta\sigma \vee \bigvee_{i=1}^n \Delta_i$ to $\S{v}$. \\
\hline
\hline

\multicolumn{2}{@{}c@{}}{\textbf{\ruleword{Eq} rule}} \\
If      & $\Gamma_1 \rightarrow \Delta_1 \vee s_1 \equals t_1 \in \S{v}$, \\
        & $s_1 \csucc{v} t_1$ and $\Delta_1 \not\csucceq{v} s_1 \equals t_1$, \\
        & $\Gamma_2 \rightarrow \Delta_2 \vee s_2 \eqorineq t_2 \in \S{v}$ with ${\eqorineq} \in \setsingle{\equals,\not\equals}$, \\
        & $s_2 \csucc{v} t_2$ and $\Delta_2 \not\csucceq{v} s_2 \eqorineq t_2$, \\
        & $\atPos{s_2}{p} = s_1$,\\
        & and $\Gamma_1 \wedge \Gamma_2 \rightarrow \Delta_1 \vee \Delta_2 \vee \replPos{s_2}{t_1}{p} \eqorineq t_2 \not\strin \S{v}$, \\
then    & add $\Gamma_1 \wedge \Gamma_2 \rightarrow \Delta_1 \vee \Delta_2 \vee \replPos{s_2}{t_1}{p} \eqorineq t_2$ to $\S{v}$. \\
\hline
\hline

\multicolumn{2}{@{}c@{}}{\textbf{\ruleword{Ineq} rule}} \\
If      & $\Gamma \rightarrow \Delta \vee t \nequals t \in \S{v}$ \\
        & and $\Gamma \rightarrow \Delta \not\strin \S{v}$, \\
then    & add $\Gamma \rightarrow \Delta$ to $\S{v}$. \\
\hline
\hline

\multicolumn{2}{@{}c@{}}{\textbf{\ruleword{Factor} rule}} \\
If      & $\Gamma \rightarrow \Delta \vee s \equals t \vee s \equals t' \in \S{v}$, \\
        & $\Delta \cup \setsingle{s \equals t} \not\csucceq{v} s \equals t'$ and $s \csucc{v} t'$ \\
        & and $\Gamma \rightarrow \Delta \vee t \nequals t' \vee s \equals t' \not\strin \S{v}$, \\
then    & add $\Gamma \rightarrow \Delta \vee t \nequals t' \vee s \equals t'$ to $\S{v}$. \\
\hline
\hline

\multicolumn{2}{@{}c@{}}{\textbf{\ruleword{Elim} rule}} \\
If      & $\Gamma \rightarrow \Delta \in \S{v}$ and \\
        & $\Gamma \rightarrow \Delta \strin \S{v} \setminus \setsingle{\Gamma \rightarrow \Delta}$ \\
then    & remove $\Gamma \rightarrow \Delta$ from $\S{v}$. \\
\hline
\hline

\multicolumn{2}{@{}c@{}}{\textbf{\ruleword{Pred} rule}} \\
If      & $\tuple{u,v,f} \in \E$, \\
        & $\bigwedge_{i=1}^l A_i \rightarrow \bigvee_{i=l+1}^{l+n} A_i \in \S{v}$, \\
        & $\Gamma_i \rightarrow \Delta_i \vee A_i\sigma \in \S{u}$ s.t.~$\Delta_i \not\csucceq{u} A_i\sigma$ for $1 \leq i \leq l$, \\
        & $A_i \in \PRtriggers$ for each $l+1 \leq i \leq l+n$, \\
        & and $\bigwedge_{i=1}^l \Gamma_i \rightarrow \bigvee_{i=1}^l \Delta_i \vee \bigvee_{i=l+1}^{l+n} A_i\sigma \not\strin \S{u}$, \\
then    & add $\bigwedge_{i=1}^l \Gamma_i \rightarrow \bigvee_{i=1}^l \Delta_i \vee \bigvee_{i=l+1}^{l+n} A_i\sigma$ to $\S{u}$, \\
where   & $\sigma = \subst{x \mapsto f(x), y \mapsto x}$. \\
\hline
\hline

\multicolumn{2}{@{}c@{}}{\textbf{\ruleword{Succ} rule}} \\
If      & $\Gamma \rightarrow \Delta \vee A \in \S{u}$ s.t.~$\Delta \not\csucceq{u} A$ and $A$ contains $f(x)$, \\
        & and, for each $A' \in K_2 \setminus \core{v}$, no edge $\tuple{u,v,f} \in \E$ \\
        & exists such that $A' \rightarrow A' \strin \S{v}$, \\
then    & let $\tuple{v,\coresym',\csuccsym'} \defeq \strategy(f, K_1, \D)$; \\
        & if $v \in \V$, then let ${\csucc{v}} \defeq {\csucc{v}} \cap {\csuccsym'}$, and \\
        & otherwise let $\V \defeq \V \cup \setsingle{v}$, \quad ${\csucc{v}} \defeq {\csuccsym'}$, \\
        & \phantom{otherwise let} $\core{v} \defeq {\coresym'}$, \quad and $\S{v} \defeq \emptyset$; \\
        & add the edge $\tuple{u,v,f}$ to $\E$; and \\
        & add $A' \rightarrow A'$ to $\S{v}$ for each $A' \in K_2 \setminus \core{v}$; \\
where   & $\sigma = \subst{x \mapsto f(x), y \mapsto x}$, \\
        & $K_1 = \setbuilder{A' \in \SUtriggers}{\top \rightarrow A'\sigma \in \S{u}}$, and \\
        & $K_2 = \{\,A' \in \SUtriggers \mid \Gamma' \rightarrow \Delta' \vee A'\sigma \in \S{u}$ and \\
        & \hspace{3cm} $\Delta' \not\csucceq{u} A'\sigma\,\}$. \\
\hline
\hline
\end{tabular}
\end{table}

\Cref{def:strategy} introduces an expansion strategy---a parameter of our
calculus that determines when and how to reuse contexts in order to satisfy
existential restrictions.

\begin{definition}\label{def:strategy}
    An \emph{expansion strategy} is a function $\strategy$ that takes a
    function symbol $f$, a set of atoms $K$, and a context structure ${\D =
    \tuple{\V,\E,\Ssym,\coresym,\csuccsym}}$. The result of $\strategy(f,
    K,\D)$ is computable in polynomial time and it is a triple
    $\tuple{v,\coresym',\csuccsym'}$ where $\coresym'$ is a subset of $K$;
    either ${v \notin \V}$ is a fresh context, or ${v \in \V}$ is an existing
    context in $\D$ such that ${\core{v} = \coresym'}$; and $\csuccsym'$ is a
    context term order.
\end{definition}

\citeA{smh14consequence-FPT} presented two basic strategies, which we can adapt
to our setting as follows.
\begin{itemize}
    \item The \emph{eager} strategy returns for each $K_1$ the context
    $v_{K_1}$ with core $K_1$. The `kind' of ground terms that $v_{K_1}$
    represents is then very specific so the set $\S{v_{K_1}}$ is likely to be
    smaller, but the number of contexts can be exponential.
    
    \item The \emph{cautious} strategy examines the function symbol $f$: if $f$
    occurs in $\Onto$ in exactly one atom of the form $B(f(x))$ and if ${B(x)
    \in K_1}$, then the result is the context $v_{B(x)}$ with core $B(x)$;
    otherwise, the result is the `trivial' context $v_\top$ with the empty
    core. Context $v_{B(x)}$ is then less constrained, but the number of
    contexts is at most linear.
\end{itemize}
\citeA{smh14consequence-FPT} discuss extensively the differences between and
the relative merits of the two strategies; although their discussion deals with
\ALCI only, their conclusions apply to \SRIQ as well.

We are now ready to show soundness and completeness.

\begin{restatable}[Soundness]{theorem}{thmsoundness}\label{theorem:soundness}
    For any expansion strategy, applying an inference rule from
    \cref{table:rules} to an ontology $\Onto$ and a context structure $\D$ that
    is sound for $\Onto$ produces a context structure that is sound for $\Onto$.
\end{restatable}

\begin{restatable}[Completeness]{theorem}{thmcompleteness}\label{theorem:completeness}
    Let $\Onto$ be an ontology, and let ${\D =
    \tuple{\V,\E,\Ssym,\coresym,\csuccsym}}$ be a context structure such that
    no inference rule from \cref{table:rules} is applicable to $\Onto$ and
    $\D$. Then, ${\InputQueryLHS \rightarrow \InputQueryRHS \strin \S{q}}$
    holds for each query clause ${\InputQueryLHS \rightarrow \InputQueryRHS}$
    and each context ${q \in \V}$ that satisfy conditions
    \ref{theorem:completeness:query}--\ref{theorem:completeness:LHS}.
    \begin{enumerate}[label={C\arabic*.},ref={C\arabic*},leftmargin=2.5em]
        \item\label{theorem:completeness:query}
        ${\Onto \models \InputQueryLHS \rightarrow \InputQueryRHS}$.

        \item\label{theorem:completeness:RHS}
        For each atom ${A \equals \TOP \in \InputQueryRHS}$ and each context
        term ${s \not\in \setsingle{x,y}}$, if ${A \csucc{q} s}$, then ${s
        \equals \TOP \in \InputQueryRHS \cup \PRtriggers}$.

        \item\label{theorem:completeness:LHS}
        For each ${A \in \InputQueryLHS}$, we have ${\InputQueryLHS \rightarrow
        A \strin \S{q}}$.
    \end{enumerate}
\end{restatable}

Conditions~\ref{theorem:completeness:RHS} and~\ref{theorem:completeness:LHS}
can be satisfied by appropriately initialising the corresponding context.
Hence, \cref{theorem:soundness,theorem:completeness} show that the following
algorithm is sound and complete for deciding ${\Onto \models \InputQueryLHS
\rightarrow \InputQueryRHS}$.
\begin{enumerate}[label={A\arabic*.},ref={A\arabic*},leftmargin=2.5em]
    \item\label{alg:initialize-D}
    Create an empty context structure $\D$ and select an expansion strategy.
    
    \item\label{alg:initialize-q}
    Introduce a context $q$ into $\D$; set ${\core{q} = \InputQueryLHS}$; for
    each ${A \in \InputQueryLHS}$, add ${\top \rightarrow A}$ to $\S{q}$ to
    satisfy condition \ref{theorem:completeness:LHS}; and initialise $\succ_q$
    in a way that satisfies condition \ref{theorem:completeness:RHS}.

    \item\label{alg:apply-rules}
    Apply the inference rules from \cref{table:rules} to $\D$ and $\Onto$.

    \item\label{alg:read}
    ${\InputQueryLHS \rightarrow \InputQueryRHS}$ holds \iff ${\InputQueryLHS
    \rightarrow \InputQueryRHS \strin \S{v}}$.
\end{enumerate}

\Cref{prop:complexity,prop:worst-case-el} show that our calculus is worst-case
optimal for both \ALCHIQ and \ELH.

\begin{restatable}{proposition}{propcomplexity}\label{prop:complexity}
    For each expansion strategy that introduces at most exponentially many
    contexts, algorithm \ref{alg:initialize-D}--\ref{alg:read} runs in
    worst-case exponential time.
\end{restatable}

\begin{restatable}{proposition}{propelh}\label{prop:worst-case-el}
    For \ELH ontologies and queries of the form ${B_1(x) \rightarrow B_2(x)}$,
    algorithm \ref{alg:initialize-D}--\ref{alg:read} runs in polynomial time
    with either the cautious or the eager strategy; and with the cautious
    strategy and the \ruleword{Hyper} rule applied eagerly, the inferences in
    step \ref{alg:apply-rules} correspond directly to the inferences of the
    \ELH calculus by~\citeA{babl05}.
\end{restatable}

\subsection{An Outline of the Completeness Proof}\label{sec:algorithm:proof-outline}

To prove \cref{theorem:completeness}, we fix an ontology $\Onto$, a context
structure $\D$, a query clause ${\InputQueryLHS \rightarrow \InputQueryRHS}$,
and a context $q$ such that properties \ref{theorem:completeness:RHS} and
\ref{theorem:completeness:LHS} of \cref{theorem:completeness} are satisfied and
${\InputQueryLHS \rightarrow \InputQueryRHS \not\strin \S{q}}$ holds, and we
construct a Herbrand interpretation that satisfies $\Onto$ but refutes
${\InputQueryLHS \rightarrow \InputQueryRHS}$. We reuse techniques from
equational theorem proving \cite{nieuwenhuis95theorem} and represent this
interpretation by a \emph{rewrite system} $\EntireModel$---a finite set of
rules of the form ${l \RuleSymbol r}$. Intuitively, such a rule says that that
any two terms of the form ${f_1(\dots f_n(l)\dots)}$ and ${f_1(\dots
f_n(r)\dots)}$ with ${n \geq 0}$ are equal, and that we can prove this equality
in one step by rewriting (i.e., replacing) $l$ with $r$. Rewrite system
$\EntireModel$ induces a Herbrand equality interpretation $\Star{\EntireModel}$
that contains each ${l \equals r}$ for which the equality between $l$ and $r$
can be verified using a finite number of such rewrite steps. The universe of
$\Star{\EntireModel}$ consists of $\F$- and $\P$-terms constructed using the
symbols in $\F$ and $\P$, and a special constant $c$; for convenience, let $\T$
be the set of all $\F$-terms from this universe.

We obtain $\EntireModel$ by unfolding the context structure $\D$ starting from
context $q$: we map each $\F$-term ${t \in \T}$ to a context $\VertexMap{t}$ in
$\D$, and we use the clauses in $\S{\VertexMap{t}}$ to construct a model
fragment $\Model{t}$---the part of $\EntireModel$ that satisfies the DL-clauses
of $\Onto$ when $x$ is mapped to $t$. The key issue is to ensure compatibility
between adjacent model fragments: when moving from a \emph{predecessor} term
$t'$ to a \emph{successor} term ${t = f(t')}$, we must ensure that adding
$\Model{t}$ to $\Model{t'}$ does not affect the truth of the DL-clauses of
$\Onto$ at term $t'$; in other words, the model fragment constructed at $t$
must respect the choices made at $t'$. We represent these choices by a ground
clause ${\QueryLHS{t} \rightarrow \QueryRHS{t}}$: conjunction $\QueryLHS{t}$
contains atoms that are `inherited' from $t'$ and so must hold at $t$, and
disjunction $\QueryRHS{t}$ contains atoms that must not hold at $t$ because
$t'$ relies on their absence.

The model fragment construction takes as parameters a term $t$, a context ${v =
\VertexMap{t}}$, and a clause ${\QueryLHS{t} \rightarrow \QueryRHS{t}}$. Let
$\GroundS{t}$ be the set of ground clauses obtained from $\S{v}$ by mapping $x$
to $t$ and $y$ to the predecessor of $t$ (if it exists), and whose body is
contained in $\QueryLHS{t}$. Moreover, let $\GroundSu{t}$ and $\GroundPr{t}$ be
obtained from $\SUtriggers$ and $\PRtriggers$ by mapping $x$ to $t$ and $y$ to
the predecessor of $t$ if one exists; thus, $\GroundSu{t}$ contains the ground
atoms of interest to the successors of $t$, and $\GroundPr{t}$ contains the
ground atoms of interest to the predecessor of $t$. The model fragment for $t$
can be constructed if properties \ref{cond:local:query}--\ref{cond:local:LHS}
hold:
\begin{enumerate}[label={L\arabic*.},ref={L\arabic*},leftmargin=2.5em]
    \item\label[condition]{cond:local:query}
    ${\QueryLHS{t} \rightarrow \QueryRHS{t} \not\strin \GroundS{t}}$.

    \item\label[condition]{cond:local:RHS}
    If ${t = c}$, then ${\QueryRHS{t} = \InputQueryRHS}$; and if ${t \neq c}$,
    then ${\QueryRHS{t} \subseteq \GroundPr{t}}$.

    \item\label[condition]{cond:local:LHS}
    For each ${A \in \QueryLHS{t}}$, we have ${\QueryLHS{t} \rightarrow A
    \strin \GroundS{t}}$.
\end{enumerate}
The construction produces a rewrite system $\Model{t}$ such that
\begin{enumerate}[label={F\arabic*.},ref={F\arabic*},leftmargin=2.5em]
    \item\label[condition]{cond:local:satisfies-clauses}
    ${\mdl{\Model{t}} \GroundS{t}}$, and

    \item\label[condition]{cond:local:not-satisfies-query}
    ${\nmdl{\Model{t}} \QueryLHS{t} \rightarrow \QueryRHS{t}}$---that is, all
    of $\QueryLHS{t}$, but none of $\QueryRHS{t}$ hold in $\Star{\Model{t}}$,
    and so the model fragment at $t$ is compatible with the `inherited'
    constraints.
\end{enumerate}
We construct rewrite system $\Model{t}$ by adapting the techniques from
paramodulation-based theorem proving. First, we order all clauses in
$\GroundS{t}$ into a sequence ${C^i = \Gamma^i \rightarrow \Delta^i \vee L^i}$,
${1 \leq i \leq n}$, that is compatible with the context ordering $\csucc{v}$
in a particular way. Next, we initialise $\Model{t}$ to $\emptyset$, and then
we examine each clause $C^i$ in this sequence; if $C^i$ does not hold in the
model constructed thus far, we make the clause true by adding $L^i$ to
$\Model{t}$. To prove \cref{cond:local:satisfies-clauses}, we assume for the
sake of a contradiction that a clause $C^i$ with smallest $i$ exists such that
${\nmdl{\Model{t}} C^i}$, and we show that an application of the \ruleword{Eq},
\ruleword{Ineq}, or \ruleword{Factor} rule to $C^i$ necessarily produces a
clause $C^j$ such that ${\nmdl{\Model{t}} C^j}$ and ${j < i}$.
\Cref{cond:local:query,cond:local:RHS,cond:local:LHS} allow us to satisfy
\cref{cond:local:not-satisfies-query}. Due to \cref{cond:local:RHS} and
\cref{cond:term:order:PR} of \cref{def:context:term:order}, we can order the
clauses in the sequence such that each clause $C^i$ capable of producing an
atom from $\QueryRHS{t}$ comes before any other clause in the sequence; and
then we use \cref{cond:local:query} to show that no such clause actually
exists. Moreover, \cref{cond:local:LHS} ensures that all atoms in
$\QueryLHS{t}$ are actually produced in $\Star{\Model{t}}$.

To obtain $\EntireModel$, we inductively unfold $\D$, and at each step we apply
the model fragment construction to the appropriate parameters. For the base
case, we map constant $c$ to context ${\VertexMap{c} = q}$, and we define
${\QueryLHS{c} = \InputQueryLHS}$ and ${\QueryRHS{c} = \InputQueryRHS}$; then,
\cref{cond:local:query,cond:local:RHS} hold by definition, and
\cref{cond:local:LHS} holds by property \ref{theorem:completeness:LHS} of
\cref{theorem:completeness}. For the induction step, we assume that we have
already mapped some term $t'$ to a context ${u = \VertexMap{t'}}$, and we
consider term $t = f(t')$ for each ${f \in \F}$.
\begin{itemize}
    \item If $t$ does not occur in an atom in $\Model{t'}$, we let ${\Model{t}
    = \setsingle{ t \RuleSymbol c}}$ and thus make $t$ equal to $c$. Term $t$
    is thus interpreted in exactly the same way as $c$, so we stop the
    unfolding.

    \item If $\Model{t'}$ contains a rule ${t \RuleSymbol s}$, then $t$ and $s$
    are equal, and so we interpret $t$ exactly as $s$; hence, we stop the
    unfolding.

    \item In all other cases, the \ruleword{Succ} rule ensures that $\D$
    contains an edge ${\tuple{u,v,f}}$ such that $v$ satisfies all
    preconditions of the rule, so we define ${\VertexMap{t} = v}$. Moreover, we
    let ${\QueryLHS{t} = \Star{\Model{t'}} \cap \GroundSu{t}}$ be the set of
    atoms that hold at $t'$ and are relevant to $t$ , and we let ${\QueryRHS{t}
    = \GroundPr{t} \setminus \Star{\Model{t'}}}$ be the set of atoms that do
    not hold at $t'$ and are relevant to $t$. We finally show that such
    $\QueryLHS{t}$ and $\QueryRHS{t}$ satisfy \cref{cond:local:query}:
    otherwise, the \ruleword{Pred} rule derives a clause in $\GroundS{t'}$ that
    is not true in $\Star{\Model{t'}}$.
\end{itemize}

After processing all relevant terms, we let $\EntireModel$ be the union of all
$\Model{t}$ from the above construction. To show that $\Star{\EntireModel}$
satisfies $\Onto$, we consider a DL-clause ${\Gamma \rightarrow \Delta \in
\Onto}$ and a substitution $\tau$ that makes the clause ground. W.l.o.g.\ we
can assume that $\tau$ is irreducible by $\EntireModel$---that is, it does not
contain terms that can we rewritten using the rules in $\EntireModel$. Since
each model fragment satisfies \cref{cond:local:not-satisfies-query}, we can
evaluate ${\Gamma\tau \rightarrow \Delta\tau}$ in $\Star{\Model{\tau(x)}}$
instead of $\Star{\EntireModel}$. Moreover, we show that
${\mdl{\Model{\tau(x)}} \Gamma\tau \rightarrow \Delta\tau}$ holds: if that were
not the case, the \ruleword{Hyper} rule derives a clause in $\GroundS{\tau(x)}$
that violates \cref{cond:local:satisfies-clauses}. Finally, we show that the
same holds for the query clause ${\InputQueryLHS \rightarrow \InputQueryRHS}$,
which completes our proof.

\section{Evaluation}\label{sec:evaluation}

\begin{figure}[t]
\centering
\begin{tikzpicture}
\pgfplotsset{every axis legend/.append style={at={(0,1)}, anchor=north west}}
\begin{groupplot}[
    ymax             = 190000,
    width            = 24.5em,
    height           = 40ex,
    ymode            = log,
    log basis y      = 10,
    axis lines       = left,
    xtick            = {100,200,300,400,500,600,700,800},
    ymajorgrids      = true,
    label style      = {font=\scriptsize},
    tick label style = {font=\scriptsize},
    ylabel           = {Classification time (ms).},
    enlargelimits    = false,
    tick align       = outside,
    legend style     = {fill=white,draw=black,at={(0.01,1)},anchor=north west,font=\scriptsize},
]
\nextgroupplot

\addplot[color=newLightBlue, mark=*, mark repeat=50, mark size=0.18em, thick] table[x=Index,y=HermitTotalTime,col sep=comma] {data/entire-corpus-data-hermit.csv};
\addlegendentry{\textcolor{newLightBlue}{HermiT}}

\addplot [color=newRed, mark=square*, mark repeat=50, mark size=0.18em, thick] table[x=Index,y=PelletTotalTime,col sep=comma] {data/entire-corpus-data-pellet.csv};
\addlegendentry{\textcolor{newRed}{Pellet}}

\addplot[color=newGreen, mark=diamond*, mark repeat=50, mark size=0.18em, thick] table[x=Index,y=FactppTotalTime,col sep=comma] {data/entire-corpus-data-factpp.csv};
\addlegendentry{\textcolor{newGreen}{FaCT++}}

\addplot[color=newOrange, mark=star, mark repeat=50, mark size=0.18em, thick] table[x=Index,y=Konclude1WTotalTime,col sep=comma] {data/entire-corpus-data-konclude.csv};
\addlegendentry{\textcolor{newOrange}{Konclude}}

\addplot[color=newPurple, mark=triangle*, mark repeat=50, mark size=0.18em, thick] table[x=Index,y=SequoiaTotalTime,col sep=comma] {data/entire-corpus-data-sequoia.csv};
\addlegendentry{\textcolor{newPurple}{Sequoia}}
\end{groupplot}
\end{tikzpicture}
\caption{Classification Times for All Ontologies}\label{fig:entire-corpus}
\end{figure}
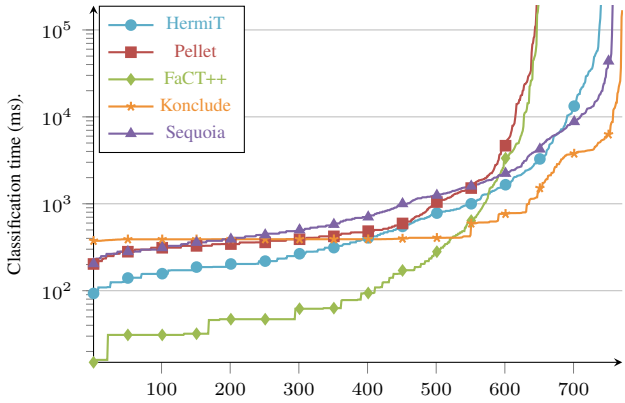

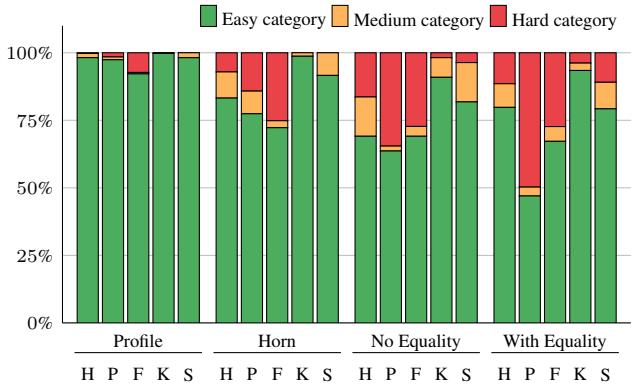
\begin{figure}
\centering
\begin{tikzpicture}
\pgfplotstableread{data/partitioned-corpus-data.txt}{\datatable}
\begin{axis}[
    ybar stacked,
    ymajorgrids            = true,
    axis lines*            = left,
    tick style             = transparent,
    width                  = 26em,
    height                 = 35ex,
    ymin                   = 0,
    xmin                   = 0,
    xmax                   = 22.5,
    bar width              = 0.8em,
    xtick                  = data,
    yticklabel             = {\pgfmathparse{\tick}\pgfmathprintnumber{\pgfmathresult}\%},
    ytick                  = {0,25,50,75,100},
    xticklabels from table = {\datatable}{Reasoner},
    xticklabel style       = {yshift=-3ex},
    label style            = {font=\scriptsize},
    tick label style       = {font=\scriptsize},
    legend style           = {fill=none,draw=none,at={(1,0.95)},anchor=south east,font=\scriptsize},
    legend columns         = -1,
    draw group line={Bucket}{Profile}{\scriptsize Profile}{-2.5ex}{0.5em},
    draw group line={Bucket}{NoProfileHorn}{\scriptsize Horn}{-2.5ex}{0.5em},
    draw group line={Bucket}{NoProfileNonHornNoEquality}{\scriptsize No Equality}{-2.5ex}{0.5em},
    draw group line={Bucket}{NoProfileNonHornWithEquality}{\scriptsize With Equality}{-2.5ex}{0.5em}
]
\addplot [fill=colorEasy!80] table [x=X, y=Easy] {\datatable};
\addlegendentry{Easy category}
\addplot [fill=colorMedium!80] table [x=X, y=Medium] {\datatable}; 
\addlegendentry{Medium category}
\addplot [fill=colorHard!80] table [x=X, y=Hard] {\datatable}; 
\addlegendentry{Hard category}
\end{axis}
\end{tikzpicture}
\caption{\label{fig:partitioned-corpus} Percentage of Easy, Medium and Hard
Ontologies per Ontology Group for HermiT (H), Pellet (P), FaCT++ (F), Konclude
(K), and Sequoia (S)}
\end{figure}

\newcommand{\ontoNr}[1]{#1\xspace}
\newcommand{\totalOntologies}{\ontoNr{777}}
\newcommand{\numberOntologiesBadRBox}{\ontoNr{7}}
\newcommand{\numHFail}{\ontoNr{42}}
\newcommand{\numPFail}{\ontoNr{138}}
\newcommand{\numFFail}{\ontoNr{132}}
\newcommand{\numKFail}{\ontoNr{8}}
\newcommand{\numSFail}{\ontoNr{22}}
\newcommand{\numSFailAllOthersPassed}{\ontoNr{8}}

We have implemented our calculus in a prototype system called Sequoia. The
calculus was implemented exactly as presented in this paper, with no
optimisation other than a suitable indexing scheme for clauses. The system is
written in Scala, and it can be used via the command line or the OWL API. It
currently handles the \SRIQ subset of OWL 2 DL (i.e., it does not support
datatypes, nominals, or reflexive roles), for which it supports ontology
classification and concept satisfiability; other standard services such as ABox
realisation are currently not supported.

We have evaluated our system using the methodology by
\citeA{DBLP:journals/ws/SteigmillerLG14} by comparing Sequoia with
HermiT~1.3.8, Pellet~2.3.1, FaCT++~1.6.4, and Konclude~1.6.1. We used all
reasoners in single-threaded mode in order to compare the underlying calculi;
moreover, Sequoia was configured to use the cautious strategy. All systems,
ontologies, and test results are available
online.\footnote{\url{http://krr-nas.cs.ox.ac.uk/2015/KR/cr/}}

We used the Oxford Ontology
Repository\footnote{\url{http://www.cs.ox.ac.uk/isg/ontologies/}} from which we
excluded \numberOntologiesBadRBox ontologies with irregular RBoxes. Since
Sequoia does not support datatypes or nominals, we have systematically replaced
datatypes and nominals with fresh classes and data properties with object
properties, and we have removed ABox assertions. We thus obtained a corpus of
\totalOntologies ontologies on which we tested all reasoners.

We run our experiments on a Dell workstation with two Intel Xeon E5-2643 v3 3.4
GHz processors with 6 cores per processor and 128 GB of RAM running Windows
Server 2012 R2. We used Java 8 update 66 with 15 GB of heap memory allocated to
each Java reasoner, and a maximum private working set size of 15 GB for each
reasoner in native code. In each test, we measured the wall-clock
classification time; this excludes parsing time for reasoners based on the OWL
API (i.e., HermiT, Pellet, FaCT++, and Sequoia). Each test was given a timeout
of 5 minutes. We report the average time over three runs, unless an exception
or timeout occurred in one of the three runs, in which case we report failure.

\Cref{fig:entire-corpus} shows an overview of the classification times for the
entire corpus. The $y$-axis shows the classification times in logarithmic
scale, and timeouts are shown as infinity. A number $n$ on the $x$-axis
represents the $n$-th easiest ontology for a reasoner with ontologies sorted
(for that reasoner) in the ascending order of classification time. For example,
a point (50, 100) on a reasoner's curve means that the 50th easiest ontology
for that reasoner took 100 ms to classify.

Sequoia could process most ontologies (733 out of 784) in under 10s, which is
consistent with the other reasoners. The system was fairly robust, failing on
only \numSFail ontologies; in contrast, HermiT failed on \numHFail, Pellet on
\numPFail, FaCT++ on \numFFail, and Konclude on \numKFail ontologies. Moreover,
Sequoia succeeded on 21 ontologies on which all of HermiT, Pellet and FaCT++
failed. Finally, there was one ontology where Sequoia succeeded and all other
reasoners failed; this was a hard version of FMA (ID $00285$) that uses both
disjunctions and number restrictions.

\Cref{fig:partitioned-corpus} shows an overview of how each reasoner performed
on each type of ontology. We partitioned the ontologies in the following four
groups: within a profile of OWL 2 DL (i.e., captured by OWL 2 EL, QL, or RL);
Horn but not in a profile; disjunctive but without number restrictions; and
disjunctive and with number restrictions. We used the OWL API to determine
profile membership, and we identified the remaining three groups after
structural transformation. In addition, for each reasoner, we categorise each
ontology as either `easy' ($<10$s), `medium' ($10$s to $5$min), and `hard'
(timeout or exception). The figure depicts a bar for each reasoner and group,
where each bar is divided into blocks representing the percentage of ontologies
in each of the aforementioned categories of difficulty. For Sequoia, over
$98\%$ of profile ontologies and over $91\%$ of out-of-profile Horn ontologies
are easy, with the remainder being of medium difficulty. Sequoia timed out
largely on ontologies containing both disjunctions and equality, and even in
this case only Konclude timed out in fewer cases.

In summary, although only an early prototype, Sequoia is a competitive reasoner
that comfortably outperforms HermiT, Pellet, and FaCT++, and which exhibits a
nice pay-as-you-go behaviour. Furthermore, problematic ontologies seem to
mostly contain complex RBoxes or large numbers in cardinality restrictions,
which suggests promising directions for future optimisation.

\section{Conclusion and Future Work}\label{sec:conclusion}

We have presented the first consequence based calculus for \SRIQ---a DL that
includes both disjunction and counting quantifiers. Our calculus combines ideas
from state of the art resolution and (hyper)tableau calculi, including the use
of ordered paramodulation for equality reasoning. Despite its increased
complexity, the calculus mimics existing calculi on \ELH ontologies. Although
it is an early prototype with plenty of room for optimisation, our system
Sequoia is competitive with well-established reasoners and it exhibits nice
pay-as-you-go behaviour in practice.

For future work, we are confident that we can extend the calculus to support
role reflexivity and datatypes, thus handling all of OWL 2 DL except nominals.
In contrast, handling nominals seems to be much more involved. In fact, adding
nominals to $\ALCHIQ$ raises the complexity of reasoning to \NExpTime so a
worst-case optimal calculus must be nondeterministic, which is quite different
from all consequence-based calculi we are aware of. Moreover, a further
challenge is to modify the calculus so that it can effectively deal with large
numbers in number restrictions.

\bibliographystyle{aaai}
\bibliography{references}

\clearpage
\appendix
\onecolumn

\section{Proof of \cref{theorem:soundness}}\label{sec:proofs}

In this chapter, we show that our calculus is sound, as stated in
\cref{theorem:soundness}. The proof is analogous to the soundness proof of
ordered superposition~\cite{nieuwenhuis95theorem}.

\thmsoundness*

\begin{proof}
Let $\Onto$ be an ontology, let ${\D = \tuple{\V,\E,\Ssym,\coresym,\csuccsym}}$
be a context structure that is sound for $\Onto$, and consider an application
of an inference rule from \cref{table:rules} to $\D$ and $\Onto$. We show that
each clause produced by the rule is a context clause and that it satisfies
\cref{soundness:cond:clauses,soundness:cond:edges} of \cref{def:contexts}.
\Cref{soundness:cond:clauses} holds obviously for the rules different from
\ruleword{Hyper}, \ruleword{Eq}, and \ruleword{Pred}. For
\cref{soundness:cond:edges}, we rely on soundness of hyperresolution: for
arbitrary formulas $\omega$, $\phi_i$, $\psi_i$, and $\gamma_i$, ${1 \leq i
\leq n}$, we have
\begin{equation}
    \{ \bigwedge_{j=1}^n \phi_j \rightarrow \omega \} \cup \bigcup_{1 \leq i \leq n} \{ \gamma_i \rightarrow \psi_i \vee \phi_i \} \models \bigwedge_{i=1}^n \gamma_i \rightarrow \bigvee_{i=1}^n \psi_i \vee \omega. \label{eqlemma:hyper:resolution}
\end{equation}
To prove the claim, we consider each rule from \cref{table:rules} and assume
that the rule is applied to clauses, contexts, and edges as shown in the table;
then, we show that the clause produced by the rule satisfies
\cref{soundness:cond:clauses} of \cref{def:contexts}; moreover, for the
\ruleword{Succ} rule, we show in addition that the edge introduced by the rule
satisfies \cref{soundness:cond:edges}.

\medskip

\noindent
(\ruleword{Core}) For each ${A \in \core{v}}$, we clearly have ${\Onto \models
\core{v} \rightarrow A}$.

\medskip

\noindent
(\ruleword{Hyper}) Since $\D$ is sound for $\Onto$, we have ${\Onto \models
\core{v} \wedge \Gamma_i \rightarrow \Delta_i \vee A_i\sigma}$ for each $i$
with ${1 \leq i \leq n}$. By \cref{eqlemma:hyper:resolution}, we have ${\Onto
\models \core{v} \wedge \bigwedge_{i=1}^n \Gamma_i \rightarrow \bigvee_{i=1}^n
\Delta_i \vee \Delta\sigma}$. Moreover, substitution $\sigma$ satisfies
${\sigma(x) = x}$, all premises are context clauses, and $\Onto$ contains only
DL-clauses; thus, the inference rule can only match an atom $S(x,z_i)$ or
$S(z_i,x)$ in an ontology clause to atoms $S(y,x)$, $S(x,y)$, $S(f(x),x)$ or
$S(x,f(x))$ in the context clause, and so $\sigma(z_i)$ is either $y$ or
$f(x)$; thus, the result is a context clause.

\medskip

\noindent
(\ruleword{Eq}) Since $\D$ is sound for $\Onto$, properties
\cref{eq:soundness:eq1:1} and \cref{eq:soundness:eq1:2} hold. Moreover, clause
in \cref{eq:soundness:eq1:3} is a logical consequence of the clauses in
\cref{eq:soundness:eq1:1} and \cref{eq:soundness:eq1:2}, so property
\cref{eq:soundness:eq1:3} holds, as required.
\begin{align}
    \Onto   & \models \core{v} \wedge \Gamma_1 \rightarrow \Delta_1 \vee s_1 \equals t_1 \label{eq:soundness:eq1:1} \\
    \Onto   & \models \core{v} \wedge \Gamma_2 \rightarrow \Delta_2 \vee s_2 \eqorineq t_2 \label{eq:soundness:eq1:2} \\
    \Onto   & \models \core{v} \wedge \Gamma_1 \wedge \Gamma_2 \rightarrow \Delta_1 \vee \Delta_2 \vee \replPos{s_2}{t_1}{p} \eqorineq t_2 \label{eq:soundness:eq1:3}
\end{align}
Finally, term $s_1$ is always of the form $g(f(x))$, term $t_1$ is of the form
$h(f(x))$ or $y$, and term $s_2$ is of the form $g(f(x))$, $B(g(f(x)))$,
$S(f(x),g(f(x)))$, or $S(g(f(x)),f(x))$; thus, $\replPos{s_2}{t_1}{p}$ is a
context term, and so the result is a context clause.

\medskip

\noindent
(\ruleword{Ineq}) Since $\D$ is sound for $\Onto$, we have ${\Onto \models
\core{v} \wedge \Gamma \rightarrow \Delta \vee t \not\equals t}$; but then, we
clearly have ${\Onto \models \core{v} \wedge \Gamma \rightarrow \Delta}$, as
required.

\medskip

\noindent
(\ruleword{Factor}) Since $\D$ is sound for $\Onto$, property
\cref{eq:soundness:factor:1} holds. Moreover, clause in
\cref{eq:soundness:factor:2} is a logical consequence of the clause in
\cref{eq:soundness:factor:1}, so property \cref{eq:soundness:factor:2} holds,
as required.
\begin{align}
    \Onto   & \models \core{v} \wedge \Gamma \rightarrow \Delta \vee s \equals t \vee s \equals t' \label{eq:soundness:factor:1} \\
    \Onto   & \models \core{v} \wedge \Gamma \rightarrow \Delta \vee t \nequals t' \vee s \equals t' \label{eq:soundness:factor:2}
\end{align}

\medskip

\noindent
(\ruleword{Elim}) The resulting context structure contains a subset of the
clauses from $\D$, so it is clearly sound for $\Onto$.

\medskip

\noindent
(\ruleword{Pred}) Let ${\sigma = \subst{x \mapsto f(x), y \mapsto x}}$. Since
$\D$ is sound for $\Onto$, properties
\cref{eq:soundness:pred:1}--\cref{eq:soundness:pred:3} hold. Now clause in
\cref{eq:soundness:pred:4} is an instance of the clause in
\cref{eq:soundness:pred:1}, so property \cref{eq:soundness:pred:4} holds. But
then, by \cref{eqlemma:hyper:resolution}, properties
\cref{eq:soundness:pred:1,eq:soundness:pred:2} imply property
\cref{eq:soundness:pred:5}. Finally, properties
\cref{eq:soundness:pred:3,eq:soundness:pred:5} imply property
\cref{eq:soundness:pred:6}, as required.
\begin{align}
    \Onto   & \textstyle\models \core{v} \wedge \bigwedge_{i=1}^m A_i \rightarrow \bigvee_{j=m+1}^{m+n} A_j \label{eq:soundness:pred:1} \\
    \Onto   & \models \core{u} \wedge \Gamma_i \rightarrow \Delta_i \vee A_i\sigma   & \text{for } 1 \leq i \leq m \label{eq:soundness:pred:2} \\
    \Onto   & \models \core{u} \rightarrow \core{v}\sigma \label{eq:soundness:pred:3} \\
    \Onto   & \textstyle\models \core{v}\sigma \wedge \bigwedge_{i=1}^m A_i\sigma \rightarrow \bigvee_{j=m+1}^{m+n} A_j\sigma \label{eq:soundness:pred:4} \\
    \Onto   & \textstyle\models \core{v}\sigma \wedge \core{v} \wedge \bigwedge_{i=1}^m \Gamma_i \rightarrow \bigvee_{j=m+1}^{m+n} A_j\sigma \label{eq:soundness:pred:5} \\
    \Onto   & \textstyle\models \core{u} \wedge \bigwedge_{i=1}^m \Gamma_i \rightarrow \bigvee_{j=m+1}^{m+n} A_j\sigma \label{eq:soundness:pred:6}
\end{align}
For each ${m+1 \leq i \leq m+n}$, we have ${A_i \in \PRtriggers}$, so $A_i$ is
of the form $B(y)$, $S(x,y)$, or $S(y,x)$; but then, the definition of $\sigma$
ensures that $A_i\sigma$ is a context atom, as required.

\medskip

\noindent
(\ruleword{Succ}) Let ${\sigma = \subst{x \mapsto f(x), y \mapsto x}}$. For
each clause ${A \rightarrow A}$ added to $\S{v}$, we clearly have ${\Onto
\models \core{v} \wedge A \rightarrow A}$, as required for
\cref{soundness:cond:clauses} of \cref{def:contexts}. Moreover, assume that the
inference rule adds an edge $\tuple{u,v,f_k}$ to $\E$; since $\D$ is sound for
$\Onto$, we have \cref{eq:soundness:succ:1}; by \cref{def:strategy}, we have
${\core{v} \subseteq K_1}$.
\begin{align}
    \Onto   & \models \core{u} \rightarrow A\sigma          & \text{for each } A \in K_1 \label{eq:soundness:succ:1} \\
    \Onto   & \models \core{u} \rightarrow \core{v}\sigma \label{eq:soundness:succ:2}
\end{align}
But then, property \cref{eq:soundness:succ:2} holds, as required for
\cref{soundness:cond:edges} of \cref{def:contexts}.
\end{proof}

\section{Preliminaries: Rewrite Systems}

In the proof of \cref{theorem:completeness} we construct a model of an
ontology, which, as is common in equational theorem proving, we represent using
a ground \emph{rewrite system}. We next recapitulate the definitions of rewrite
systems, following the presentation by \citeA{baader98term}.

Let $\T$ be the set of all ground terms constructed using a distinguished
constant $c$ (of sort $\F$), the function symbols from $\F$, and the predicate
symbols from $\P$. A (ground) \emph{rewrite system} $R$ is a binary relation on
$\T$. Each pair ${(s,t) \in {R}}$ is called a \emph{rewrite rule} and is
commonly written as ${s \RuleSymbol t}$. The \emph{rewrite relation}
$\RwRelSymbol{R}$ for $R$ is the smallest binary relation on $\T$ such that,
for all terms ${s_1, s_2, t \in \T}$ and each (not necessarily proper) position
$p$ in $t$, if $\Rule{s_1}{R}{s_2}$, then
$\RwRel{\replPos{t}{s_1}{p}}{R}{\replPos{t}{s_2}{p}}$. Moreover,
$\RwRelRefTransSymbol{R}$ is the reflexive--transitive closure of
$\RwRelSymbol{R}$, and $\CongruenceSymbol{R}$ is the
reflexive--symmetric--transitive closure of $\RwRelSymbol{R}$. A term $s$ is
\emph{irreducible by} $R$ if no term $t$ exists such that $\RwRel{s}{R}{t}$;
and a literal, clause, or substitution $\alpha$ is \emph{irreducible by} $R$ if
no term occurring in $\alpha$ is irreducible by $R$. Moreover, term $t$ is a
\emph{normal form} of $s$ w.r.t.\ $R$ if ${\Congruence{s}{R}{t}}$ and $t$ is
irreducible by $R$. We consider the following properties of rewrite systems.
\begin{itemize}
    \item $R$ is \emph{terminating} if no infinite sequence ${s_1, s_2, \dots}$
    of terms exists such that, for each~$i$, we have
    ${\RwRel{s_i}{R}{s_{i+1}}}$.

    \item $R$ is \emph{left-reduced} if, for each ${\Rule{s}{R}{t}}$, the term
    $s$ is irreducible by ${R \setminus \setsingle{s \RuleSymbol t}}$.

    \item $R$ is \emph{Church-Rosser} if, for all terms $t_1$ and $t_2$ such
    that ${\Congruence{t_1}{R}{t_2}}$, a term $z$ exists such that
    ${\RwRelRefTrans{t_1}{R}{z}}$ and ${\RwRelRefTrans{t_2}{R}{z}}$.
\end{itemize}
If $R$ is terminating and left-reduced, then $R$ is Church-Rosser
\cite[Theorem~2.1.5 and Exercise 6.7]{baader98term}. If $R$ is Church-Rosser,
then each term $s$ has a unique normal form $t$ such that ${s
\RwRelRefTransSymbol{R} t}$ holds. The \emph{Herbrand interpretation induced
by} a Church-Rosser system $R$ is the set $\Star{R}$ such that, for all ${s,t
\in \T}$, we have ${s \equals t \in \Star{R}}$ \iff $\Congruence{s}{R}{t}$.

Term orders can be used to prove termination of rewrite systems. A term order
$\succ$ is a \emph{simplification order} if the following conditions hold:
\begin{itemize}
    \item for all terms $s_1$, $s_2$, and $t$, all positions $p$ in $t$, and
    all substitutions $\sigma$, we have that ${s_1 \succ s_2}$ implies
    ${\replPos{t}{s_1\sigma}{p} \succ \replPos{t}{s_2\sigma}{p}}$; and

    \item for each term $s$ and each proper position $p$ in $s$, we have ${s
    \succ \atPos{s}{p}}$.
\end{itemize}
Given a rewrite system $R$, if a simplification order $\succ$ exists such that
${s \RuleSymbol t \in R}$ implies ${s \succ t}$, then $R$ is terminating
\cite[Theorems~5.2.3 and~5.4.8]{baader98term}, and ${s \RwRelSymbol{R} t}$
implies ${s \succ t}$.

\section{Proof of \cref{theorem:completeness}}\label{sec:proofs:completeness}

\thmcompleteness*

In this section, we fix an ontology $\Onto$, a context structure ${\D =
\tuple{\V,\E,\Ssym,\coresym,\csuccsym}}$, a context ${q \in \V}$, and a query
clause ${\InputQueryLHS \rightarrow \InputQueryRHS}$ such that
conditions~\ref{theorem:completeness:LHS} and~\ref{theorem:completeness:RHS} of
\cref{theorem:completeness} are satisfied, and we show the contrapositive of
condition \ref{theorem:completeness:query}: if ${\InputQueryLHS \rightarrow
\InputQueryRHS \not\strin \S{q}}$, then ${\Onto \not\models \InputQueryLHS
\rightarrow \InputQueryRHS}$. To this end, we construct a rewrite system
$\EntireModel$ such that the induced Herbrand model $\Star{\EntireModel}$
satisfies all clauses in $\Onto$, but not ${\InputQueryLHS \rightarrow
\InputQueryRHS}$. We construct the model using a distinguished constant $c$,
the unary function symbols from $\F$, and the unary and binary predicate
symbols from $\P_1$ and $\P_2$, respectively.

Let $t$ be a term. If $t$ is of the form ${t = f(s)}$, then $s$ is the
\emph{predecessor} of $t$, and $t$ is a \emph{successor} of $s$; by these
definitions, a constant has no predecessor. The $\F$-\emph{neighbourhood} of
$t$ is the set of $\F$-terms containing $t$, $f(t)$ with ${f \in \F}$, and the
predecessor $t'$ of $t$ if one exists; the $\P$-\emph{neighbourhood} of $t$
contains $\P$-terms $B(t)$, $S(t,f(t))$, $S(f(t),t)$, $B(f(t))$, and, if $t$
has the predecessor $t'$, also $\P$-terms $S(t',t)$, $S(t,t')$, and $B(t')$,
for all ${B \in \P_1}$ and ${S \in \P_2}$. Let $\Grounding{t}$ be the
substitution such that ${\Grounding{t}(x) = t}$ and, if $t$ has the predecessor
$t'$, then ${\Grounding{t}(y) = t'}$. Finally, for each term $t$, we define
sets of atoms $\GroundPr{t}$ and $\GroundSu{t}$ as follows:
\begin{align}
    \GroundSu{t} & = \setbuilder{A\Grounding{t}}{A \in \SUtriggers \text{ and } A\Grounding{t} \text{ is ground}} \label{eq:ground:SU} \\
    \GroundPr{t} & = \setbuilder{A\Grounding{t}}{A \in \PRtriggers \text{ and } A\Grounding{t} \text{ is ground}} \label{eq:ground:PR}
\end{align}

\subsection{Constructing a Model Fragment}\label{sec:proofs:completeness:fragment}

In this section, we show how, given a term $t$, we can generate a part of the
model of $\Onto$ that covers the neighbourhood of $t$. In the rest of
\cref{sec:proofs:completeness:fragment}, we fix the following parameters to the
model fragment generation process:
\begin{itemize}
    \item $t$ is a ground $\F$-term,

    \item $v$ is a context in $\D$,

    \item $\QueryLHS{t}$ is a conjunction of atoms, and

    \item $\QueryRHS{t}$ is a disjunction of atoms.
\end{itemize}
Let $\GroundS{t}$ be the set of ground clauses obtained from $\S{v}$ as follows:
\begin{displaymath}
    \GroundS{t} = \{ \Gamma\Grounding{t} \rightarrow \Delta\Grounding{t} \mid \Gamma \rightarrow \Delta \in \S{v}, \text{ both } \Gamma\Grounding{t} \text{ and } \Delta\Grounding{t} \text{ are ground, and } \Gamma\Grounding{t} \subseteq \QueryLHS{t} \}
\end{displaymath}
We assume that the following conditions hold.
\begin{enumerate}[label={L\arabic*.},ref={L\arabic*},leftmargin=2.5em]
    \item
    ${\QueryLHS{t} \rightarrow \QueryRHS{t} \not\strin \GroundS{t}}$.

    \item
    If ${t = c}$, then ${\QueryRHS{t} = \InputQueryRHS}$; and if ${t \neq c}$,
    then ${\QueryRHS{t} \subseteq \GroundPr{t}}$.

    \item
    For each ${A \in \QueryLHS{t}}$, we have ${\QueryLHS{t} \rightarrow A
    \strin \GroundS{t}}$.
\end{enumerate}

We next construct a rewrite system $\Model{t}$ such that ${\Star{\Model{t}}
\models \GroundS{t}}$ and ${\Star{\Model{t}} \not\models \QueryLHS{t}
\rightarrow \QueryRHS{t}}$ holds. Throughout
\cref{sec:proofs:completeness:fragment}, we treat the terms in the
$\F$-neighbourhood of $t$ as if they were constants. Thus, even though the
rewrite system $R$ will contain terms $t$ and $f(t)$, we will not consider
terms with further nesting.

\subsubsection{Grounding the Context Order}\label{sec:proofs:completeness:fragment:order}

To construct $\Model{t}$, we need an order on the terms in the neighbourhood of
$t$ that is compatible with $\csucc{v}$. To this end, let $\gsucc{t}$ be a
total, strict, simplification order on the set of ground terms constructed
using the $\F$-neighbourhood of $t$ and the predicate symbols in $\P$ that
satisfies the following conditions for all context terms $s_1$ and $s_2$ such
that $s_1\Grounding{t}$ and $s_2\Grounding{t}$ are both ground, and where $t'$
is the predecessor of $t$ if it exists.
\begin{enumerate}[label={O\arabic*.},ref={O\arabic*},leftmargin=2.5em]
    \item\label[condition]{cond:order:compatibility}
    ${s_1 \csucc{v} s_2}$ implies ${s_1\Grounding{t} \gsucc{t}
    s_2\Grounding{t}}$.

    \item\label[condition]{cond:order:query}
    ${s_1\Grounding{t} \equals \TOP \in \QueryRHS{t}}$ and ${s_1\Grounding{t}
    \gsucc{t} s_2\Grounding{t}}$ and ${s_2\Grounding{t} \not\in
    \setsingle{t,t'}}$ imply ${s_2\Grounding{t} \equals \TOP \in \QueryRHS{t}}$.
\end{enumerate}
Condition \ref{theorem:completeness:RHS} of \cref{theorem:completeness} and
\cref{cond:term:order:PR} of \cref{def:context:term:order} ensure that the
order $\csucc{v}$ on (nonground) context terms can be grounded in a way
compatible with \cref{cond:local:RHS}. Moreover, since in this section we treat
all $\F$-terms as constants, we can make the $\P$-terms of the form $B(t')$,
$S(t',t)$, and $S(t,t')$ smaller than other $\F$- and $\P$-terms (i.e., we do
not need to worry about defining the order on the predecessor of $t'$ or on the
ancestors of $f(t)$). Thus, at least one such order exists, so in the rest of
this section we fix an arbitrary such order $\gsucc{t}$. We extend $\gsucc{t}$
to ground literals (also written $\gsucc{t}$) by identifying each ${s \nequals
t}$ with the multiset ${\{ s, s, t, t \}}$ and each ${s \equals t}$ with the
multiset ${\{ s, t \}}$, and then comparing the result using the multiset
extension of the term order (as defined in \cref{sec:preliminaries}). Finally,
we further extend $\gsucc{t}$ to disjunctions of ground literals (also written
$\gsucc{t}$) by identifying each disjunction $\bigvee_{i=1}^n L_i$ with the
multiset ${\{L_1,\ldots,L_n\}}$ and then comparing the result using the
multiset extension of the literal order.

\subsubsection{Constructing the Rewrite System $\Model{t}$}\label{sec:proofs:completeness:fragment:construction}

We arrange all clauses in $\GroundS{t}$ into a sequence ${C^1, \dots, C^n}$. Since
the body of each $C^i$ is a subset of $\QueryLHS{t}$, no
$C^i$ can contain $\bot$ in its head as that would contradict
\cref{cond:local:query}; thus, we can assume that each $C^i$ is of the form
${C^i = \Gamma^i \rightarrow \Delta^i \vee L^i}$ where ${L^i \gsucc{t}
\Delta^i}$, literal $L^i$ is of the form ${L^i = l^i \eqorineq r^i}$ with ${{\eqorineq}
\in \setsingle{{\equals}, {\nequals}}}$, and ${l^i \gsucceq{t} r^i}$. For the
rest of \cref{sec:proofs:completeness:fragment}, we reserve $C^i$, $\Gamma^i$,
$\Delta^i$, $L^i$, $l^i$, and $r^i$ for referring to the (parts of) the clauses
in this sequence. Finally, we assume that, for all ${1 \leq i<j \leq n}$, we
have ${\Delta^j \vee L^j \gsucceq{t} \Delta^i \vee L^i}$.

We next define the sequence ${\Model{t}^0, \dots, \Model{t}^n}$ of rewrite
systems by setting ${\Model{t}^0 = \emptyset}$ and defining each $\Model{t}^i$
with ${1 \leq i \leq n}$ inductively as follows:
\begin{itemize}
    \item $\Model{t}^i = \Model{t}^{i-1} \cup \setsingle{l^i \RuleSymbol r^i}$
    if $L^i$ is of the form $l^i \equals r^i$ such that
    \begin{enumerate}[label={R\arabic*.},ref={R\arabic*},leftmargin=2.5em]
        \item\label[condition]{cond:rule:not:models}
        $\nmdl{(\Model{t}^{i-1})} \Delta^i \vee l^i \equals r^i$,

        \item\label[condition]{cond:rule:order}
        $l^i \gsucc{t} r^i$,

        \item\label[condition]{cond:rule:irreducible}
        $l^i$ is irreducible by $\Model{t}^{i-1}$, and

        \item\label[condition]{cond:rulegen}
        $s \equals r^i \not\in \Star{(\Model{t}^{i-1})}$ for each $l^i \equals
        s \in \Delta^i$;
    \end{enumerate}

    \medskip

    \item $\Model{t}^i = \Model{t}^{i-1}$ in all other cases.
\end{itemize}
Finally, let ${\Model{t} = \Model{t}^n}$; we call $\Model{t}$ the \emph{model
fragment for} $t$, $v$, $\QueryLHS{t}$, and $\QueryRHS{t}$. Each clause ${C^i =
\Gamma^i \rightarrow \Delta^i \vee l^i \equals r^i}$ that satisfies the first
condition in the above construction is called \emph{generative}, and the clause
is said to \emph{generate} the rule ${l^i \RuleSymbol r^i}$ in $\Model{t}$.

\subsubsection{The Properties of the Model Fragment $\Model{t}$}

\begin{lemma}
    The rewrite system $\Model{t}$ is Church-Rosser.
\end{lemma}

\begin{proof}
To see that $\Model{t}$ is terminating, simply note that, for each rule
${\Rule{l}{\Model{t}}{r}}$, \cref{cond:rule:order} ensures ${l \gsucc{t} r}$,
and that $\gsucc{t}$ is a simplification order.

To see that $\Model{t}$ is left-reduced, consider an arbitrary rule
${\Rule{l}{\Model{t}}{r}}$ that is added to $\Model{t}$ in step $i$ of the
clause sequence. By \cref{cond:rule:irreducible}, ${l \RuleSymbol r}$ is
irreducible by $\Model{t}^i$. Now consider an arbitrary rule
${\Rule{l'}{\Model{t}}{r'}}$ that is added to $\Model{t}$ at any step $j$ of
the construction where ${j > i}$. The definition of the clause order implies
${l' \equals r' \gsucceq{t} l \equals r}$; since ${l' \gsucc{t} r'}$ and ${l
\gsucc{t} r}$ by \cref{cond:rule:order}, by the definition of the literal order
we have ${l' \gsucceq{t} l}$. Since ${\Rule{l}{\Model{t}^{j-1}}{r}}$,
\cref{cond:rule:irreducible} ensures ${l \neq l'}$, and so we have ${l'
\gsucc{t} l}$; consequently, $l'$ is not a subterm of $l$, and thus $l$ is
irreducible by $\Model{t}^j$.
\end{proof}

\begin{lemma}\label{lemma:inequality:stays}
    For each ${1 \leq i \leq n}$ and each ${l \nequals r \in \Delta^i \vee
    L^i}$, we have ${\mdl{(\Model{t}^{i-1})} l \equals r}$ \iff
    ${\mdl{\Model{t}} l \equals r}$.
\end{lemma}

\begin{proof}
Consider an arbitrary clause ${C^i = \Gamma^i \rightarrow \Delta^i \vee L^i}$
and an arbitrary inequality ${l \nequals r \in \Delta^i \vee L^i}$. If ${l
\equals r \in \Star{(\Model{t}^{i-1})}}$, then ${\Model{t}^{i-1} \subseteq
\Model{t}}$ implies ${l \equals r \in \Star{\Model{t}}}$, and so we have
${\mdl{\Model{t}} l \equals r}$, as required. Now assume that ${l \equals r
\not\in \Star{(\Model{t}^{i-1})}}$. Let $l'$ and $r'$ be the normal forms of
$l$ and $r$, respectively, w.r.t.\ $\Model{t}^{i-1}$. Now consider an arbitrary
$j$ with ${i \leq j \leq n}$ such that ${l^j \RuleSymbol r^j}$ is generated by
$C^j$. We then have ${l^j \equals r^j \gsucc{t} l \nequals r}$, which by the
definition of literal order implies ${l^j \gsucc{t} l \gsucceq{t} l'}$ and ${l^j
\gsucc{t} r \gsucceq{t} r'}$; since $\gsucc{t}$ is a simplification order, $l^j$
is a subterm of neither $l'$ nor $r'$. Thus, $l'$ and $r'$ are the normal forms
of $l$ and $r$, respectively, w.r.t.\ $\Model{t}^j$, and so we have ${l'
\equals r' \not\in \Star{(\Model{t}^j)}}$; but then, we have ${l \equals r
\not\in \Star{(\Model{t}^j)}}$, as required.
\end{proof}

\begin{lemma}\label{lemma:generative:clauses}
    For each generative clause ${\Gamma^i \rightarrow \Delta^i \vee l^i \equals
    r^i}$, we have ${\nmdl{\Model{t}} \Delta^i}$.
\end{lemma}

\begin{proof}
Consider a generative clause ${C^i = \Gamma^i \rightarrow \Delta^i \vee l^i
\equals r^i}$ and a literal ${L \in \Delta^i}$; \cref{cond:rule:not:models}
ensures that ${\nmdl{(\Model{t}^{i-1})} L}$. We next show that
${\nmdl{(\Model{t}^{i-1})} L}$.

Assume that $L$ is of the form ${l \nequals r}$. Since ${l \nequals r \in
\Delta^i \vee l^i \equals r^i}$, by \cref{lemma:inequality:stays} we have
${\nmdl{\Model{t}} L}$, as required.

Assume that $L$ is of the form ${l \equals r}$ with ${l \gsucc{t} r}$. We show
by induction that, for each $j$ with ${i \leq j \leq n}$, we have
${\nmdl{(\Model{t}^j)} L}$. To this end, we assume that
${\nmdl{(\Model{t}^{j-1})} L}$. If $C^j$ is not generational, then
${\Model{t}^j = \Model{t}^{j-1}}$, and so ${\nmdl{(\Model{t}^j)} L}$. Thus,
assume that $C^j$ is generational. We consider the following two cases.
\begin{itemize}
    \item ${l^j = l}$. We have the following two subcases.
    \begin{itemize}
        \item ${j = i}$. \Cref{cond:rulegen} then ensures ${r \equals r^i \not\in
        \Star{(\Model{t}^{i-1})}}$. Let $r'$ and $r''$ be the normal forms of
        $r$ and $r^i$, respectively, w.r.t.\ $\Model{t}^{i-1}$; we have ${r'
        \equals r'' \not\in \Star{(\Model{t}^{i-1})}}$. Moreover, ${l \gsucc{t}
        r \gsucceq{t} r'}$ and ${l \gsucc{t} r^i \gsucceq{t} r''}$ hold; since
        $\gsucc{t}$ is a simplification order, $l$ is a subterm of neither $r'$
        nor $r''$; therefore, $r'$ and $r''$ are the normal forms of $r$ and
        $r^i$, respectively, w.r.t.\ $\Model{t}^i$, and therefore ${r' \equals
        r'' \not\in \Star{(\Model{t}^i)}}$. Finally, since
        ${\Rule{l}{\Model{t}^i}{r^i}}$, term $r''$ is the normal form of $l$
        w.r.t.\ $\Model{t}^i$, and so ${l \equals r \not\in
        \Star{(\Model{t}^i)}}$.

        \item ${j > i}$. But then, ${l^j \equals r^j \gsucceq{t} l^i \equals r^i
        \gsucc{t} l \equals r}$ implies ${l^j = l^i = l}$. Furthermore, $C^i$ is
        generational, so we have ${\Rule{l^i}{\Model{t}^{j-1}}{r^i}}$. But
        then, $l^j$ is not irreducible by $\Model{t}^{j-1}$, which contradicts
        \cref{cond:rule:irreducible}.
    \end{itemize}

    \item ${l^j \gsucc{t} l}$. Let $l'$ and $r'$ be the normal forms of $l$ and
    $r$, respectively, w.r.t.\ $\Model{t}^{j-1}$. Then, we have ${l^j \gsucc{t}
    l \gsucceq{t} l'}$ and ${l^j \gsucc{t} r \gsucceq{t} r'}$; since
    $\gsucc{t}$ is a simplification order, $l^j$ is a subterm of neither $l'$
    nor $r'$. Thus, $l'$ and $r'$ are the normal forms of $l$ and $r$,
    respectively, w.r.t.\ $\Model{t}^j$, and so ${l' \equals r' \not\in
    \Star{(\Model{t}^j)}}$; hence, ${l \equals r \not\in \Star{(\Model{t}^j)}}$
    holds. \qedhere
\end{itemize}
\end{proof}

\begin{lemma}\label{lemma:all:true:before}
    Let ${\Gamma \rightarrow \Delta}$ be a clause with ${\Gamma \rightarrow
    \Delta \strin \GroundS{t}}$. Then ${\mdl{\Model{t}} \Delta}$ holds if $i$
    with ${1 \leq i \leq n+1}$ exists such that
    \begin{enumerate}
        \item\label[condition]{lemma:all:true:before:true}
        for each ${1 \leq j < i}$, we have ${\mdl{\Model{t}} \Delta^j \vee
        L^j}$, and

        \item\label[condition]{lemma:all:true:before:clauses}
        if ${i \leq n}$ (i.e., $i$ is an index of a clause from $\GroundS{t}$),
        then ${\Delta^i \vee L^i \gsucc{t} \Delta}$.
    \end{enumerate}
\end{lemma}

\begin{proof}
Assume that ${\Gamma \rightarrow \Delta \strin \GroundS{t}}$ holds. If ${\Gamma
\rightarrow \Delta}$ satisfies \cref{cond:redundancy:tautology} of
\cref{def:redundancy}, then we clearly have ${\mdl{\Model{t}} \Delta}$. Assume
that ${\Gamma \rightarrow \Delta}$ satisfies
\cref{cond:redundancy:strengthening} of \cref{def:redundancy} due to some
clause ${\Gamma^j \rightarrow \Delta^j \vee L^j \in \GroundS{t}}$ such that
${\Gamma^j \subseteq \Gamma}$ and ${\Delta^j \cup \setsingle{L^j} \subseteq
\Delta}$ hold; the latter clearly implies ${\Delta \gsucceq{t} \Delta^j \vee
L^j}$. Let $i$ be an integer satisfying this lemma's assumption. If ${i =
n+1}$, then we clearly have ${j < i}$; otherwise, ${\Delta^i \vee L^i \gsucc{t}
\Delta}$ implies ${\Delta^i \vee L^i \gsucc{t} \Delta^j \vee L^j}$, and so we
also have ${j < i}$. But then, by the lemma assumption we have
${\mdl{\Model{t}} \Delta^j \vee L^j}$, which implies ${\mdl{\Model{t}}
\Delta}$, as required.
\end{proof}

\begin{lemma}\label{lemma:lift-redundancy}
    For each clause ${\Gamma \rightarrow \Delta}$ such that ${\Gamma
    \rightarrow \Delta \strin \S{v}}$ and ${\Gamma\Grounding{t} \subseteq
    \QueryLHS{t}}$ hold, we have ${\Gamma\Grounding{t} \rightarrow
    \Delta\Grounding{t} \strin \GroundS{t}}$.
\end{lemma}

\begin{proof}
Assume that ${\Gamma \rightarrow \Delta \strin \S{v}}$ holds. If ${\Gamma
\rightarrow \Delta}$ satisfies \cref{cond:redundancy:tautology} of
\cref{def:redundancy}, then terms $s$ and $s'$ exist such that ${s \equals s
\in \Delta}$ or ${\{ s \equals s', \; s \nequals s' \} \subseteq \Delta}$; but
then, ${s\Grounding{t} \equals s'\Grounding{t} \in \Delta\Grounding{t}}$ or
${\{ s\Grounding{t} \equals s'\Grounding{t}, \; s\Grounding{t} \nequals
s'\Grounding{t} \} \subseteq \Delta\Grounding{t}}$, so ${\Gamma\Grounding{t}
\rightarrow \Delta\Grounding{t} \strin \GroundS{t}}$ holds. Furthermore, if
${\Gamma \rightarrow \Delta}$ satisfies \cref{cond:redundancy:strengthening} of
\cref{def:redundancy}, then clause ${\Gamma' \rightarrow \Delta' \in \S{v}}$
exists such that ${\Gamma' \subseteq \Gamma}$ and ${\Delta' \subseteq \Delta}$;
but then, due to ${\Gamma'\Grounding{t} \subseteq \Gamma\Grounding{t} \subseteq
\QueryLHS{t}}$, we have that ${\Gamma'\Grounding{t} \rightarrow
\Delta'\Grounding{t} \in \GroundS{t}}$ holds, and so ${\Gamma\Grounding{t}
\rightarrow \Delta\Grounding{t} \strin \GroundS{t}}$ holds as well.
\end{proof}

\begin{lemma}\label{lemma:interpretation:satisfies:clause}
    For each ${\Gamma \rightarrow \Delta \in \GroundS{t}}$, we have
    ${\mdl{\Model{t}} \Delta}$.
\end{lemma}

\begin{proof}
For the sake of a contraction, choose ${C^i = \Gamma^i \rightarrow \Delta^i
\vee L^i}$ as the clause in the sequence of clauses from
\cref{sec:proofs:completeness:fragment:construction} with the smallest $i$ such
that ${\nmdl{\Model{t}} \Delta^i \vee L^i}$; please recall that ${L^i \gsucc{t}
\Delta^i}$ and that ${L^i = l^i \eqorineq r^i}$ with ${\eqorineq} \in \{
{\equals}, {\nequals} \}$. Due to our choice of $i$,
\cref{lemma:all:true:before:true} of \cref{lemma:all:true:before} holds for
$C^i$ and $i$. By the definition of $\GroundS{t}$, a clause ${\Gamma
\rightarrow \Delta \vee L \in \S{v}}$ exists such that
\begin{align}
    \Gamma\Grounding{t} = \Gamma^i \subseteq \QueryLHS{t}, \quad \Delta\Grounding{t} = \Delta^i, \quad L\Grounding{t} = L^i, \quad \text{and} \quad \Delta \not\csucceq{v} L. \label{eq:nonground:premise}
\end{align}
We next prove the claim of this lemma by considering the possible forms of $L^i$.

\medskip

Assume ${L^i = l^i \equals r^i}$ with ${l^i = r^i}$. But then, we have
${\mdl{\Model{t}} L^i}$, which contradicts our assumption that
${\nmdl{\Model{t}} \Delta^i \vee L^i}$.

\medskip

Assume ${L^i = l^i \equals r^i}$ with ${l^i \gsucc{t} r^i}$. Then, literal $L$
is of the form ${l \equals r}$ such that ${l\Grounding{t} \equals
r\Grounding{t} = l^i \equals r^i}$. By the definition of $\gsucc{t}$, we have
${l \csucc{v} r}$. We first show that ${\nmdl{(\Model{t}^{i-1})} \Delta^i \vee
L^i}$ holds; towards this goal, note that, for each equality ${s_1 \equals s_2
\in \Delta^i \vee L^i}$, properties ${\nmdl{\Model{t}} s_1 \equals s_2}$ and
${\Model{t}^{i-1} \subseteq \Model{t}}$ imply ${\nmdl{(\Model{t}^{i-1})} s_1
\equals s_2}$; and for each inequality ${s_1 \nequals s_2 \in \Delta^i}$,
\cref{lemma:inequality:stays} and ${\nmdl{\Model{t}} s_1 \nequals s_2}$ imply
${\nmdl{(\Model{t}^{i-1})} s_1 \nequals s_2}$. Thus, clause $C^i$ satisfies
\cref{cond:rule:not:models,cond:rule:order}; however, since ${\nmdl{\Model{t}}
l^i \equals r^i}$, clause $C^i$ is not generational and thus either
\cref{cond:rule:irreducible} or \cref{cond:rulegen} are not satisfied. We next
consider both of these possibilities.
\begin{itemize}
    \item \Cref{cond:rule:irreducible} does not hold---that is, $l^i$ is
    reducible by $\Model{t}^{i-1}$. By the definition of reducibility, a
    position $p$ and a clause ${C^j = \Gamma^j \rightarrow \Delta^j
    \vee l^j \equals r^j}$ generating the rule ${l^j \RuleSymbol r^j}$ exist such that ${j < i}$ and ${\atPos{l^i}{p} =
    l^j}$. Due to ${j < i}$, we have ${l^i \equals r^i \gsucceq{t} l^j \equals
    r^j}$; together with ${l^j \equals r^j \gsucc{t} \Delta^j}$, we have ${l^i
    \equals r^i \gsucc{t} \Delta^j}$. \cref{lemma:generative:clauses} ensures
    ${\nmdl{\Model{t}} \Delta^j}$, and the definition of $\GroundS{t}$ ensures
    that a clause ${\Gamma' \rightarrow \Delta' \vee l' \equals r' \in \S{v}}$
    exists such that
    \begin{align}
         \Gamma'\Grounding{t} = \Gamma^j \subseteq \QueryLHS{t}, \quad \Delta'\Grounding{t} = \Delta^j, \quad  l'\Grounding{t} =l^j, \quad r'\Grounding{t} = r^j, \quad \Delta' \not\csucceq{v} l' \equals r', \quad \text{and} \quad l' \csucc{v} r'.  \label{eq:generational:premise}
    \end{align}
    By the assumption of \cref{theorem:completeness}, the \ruleword{Eq} rule is
    not applicable to \cref{eq:nonground:premise,eq:generational:premise}, and
    so ${\Gamma \wedge \Gamma' \rightarrow \Delta \vee \Delta' \vee
    \replPos{l}{r'}{p} \equals r \strin \S{v}}$. Let ${\Delta'' = \Delta^i \vee
    \Delta^j \vee \replPos{l^i}{r^j}{p} \equals r^i}$. Then clearly
    ${\Gamma\Grounding{t} \cup \Gamma'\Grounding{t} \subseteq \QueryLHS{t}}$,
    so \cref{lemma:lift-redundancy} ensures that ${\Gamma^i \wedge \Gamma^j
    \rightarrow \Delta'' \strin \GroundS{t}}$ holds. Set $\Star{\Model{t}}$ is
    a congruence, so ${\replPos{l^i}{r^j}{p} \equals r^i \not\in
    \Star{\Model{t}}}$ holds, and therefore ${\nmdl{\Model{t}} \Delta''}$
    holds. Finally, $\gsucc{t}$ is a simplification order, which ensures ${l^i
    \equals r^i \gsucc{t} \replPos{l^i}{r^j}{p} \equals r^i}$; together with
    ${l^i \equals r^i \gsucc{t} \Delta^i}$ and ${l^i \equals r^i \gsucc{t}
    \Delta^j}$, we have ${l^i \equals r^i \gsucc{t} \Delta''}$. But then,
    \cref{lemma:all:true:before} implies ${\mdl{\Model{t}} \Delta''}$, which is
    a contradiction.

    \item \Cref{cond:rulegen} does not hold. Then, some term $s$ exists such
    that ${l^i \equals s \in \Delta^i}$ and ${s \equals r^i \in
    \Star{(\Model{t}^{i-1})}}$. Due to ${\Model{t}^{i-1} \subseteq \Model{t}}$,
    we have ${s \equals r^i \in \Star{\Model{t}}}$, and so ${\nmdl{\Model{t}} s
    \nequals r^i}$. Furthermore, ${\Delta \vee L}$ is of the form ${\Delta'
    \vee l \equals r \vee l' \equals r'}$ such that
     \begin{align}
         l\Grounding{t} = l^i, \quad r\Grounding{t} = s, \quad l'\Grounding{t} = l^i, \quad \text{and} \quad r'\Grounding{t} = r^i.
     \end{align}
    But then, we clearly have ${l' = l}$. By the assumption of
    \cref{theorem:completeness}, the \ruleword{Factor} rule is not applicable
    to ${\Gamma \rightarrow \Delta \vee L}$, and so we have ${\Gamma
    \rightarrow \Delta' \vee r \nequals r' \vee l' \equals r' \strin \S{v}}$.
    Let ${\Delta'' = \Delta'\Grounding{t} \vee s \nequals r^i \vee l^i \equals
    r^i}$. But then, ${\Gamma\Grounding{t} \subseteq \QueryLHS{t}}$ and
    \cref{lemma:lift-redundancy} ensure that ${\Gamma^i \rightarrow \Delta''
    \strin \GroundS{t}}$ holds. By all the previous observations, we have
    ${\nmdl{\Model{t}} \Delta''}$. Moreover, ${l^i \gsucc{t} r^i}$ and ${l^i
    \gsucc{t} s}$ imply ${l^i \equals r^i \gsucc{t} s \equals r^i}$; thus,
    ${\Delta^i \vee l^i \equals r^i \gsucc{t} \Delta''}$ holds. But then,
    \cref{lemma:all:true:before} implies ${\mdl{\Model{t}} \Delta''}$, which is
    a contradiction.
\end{itemize}

\medskip

Assume ${L^i = l^i \nequals r^i}$ with ${l^i = r^i}$. Then, literal $L$ is of
the form ${l \nequals r}$ such that ${l\Grounding{t} \nequals r\Grounding{t} =
l^i \nequals r^i}$. But then, ${l^i = r^i}$ implies ${l = r}$. By the
assumption of \cref{theorem:completeness}, the \ruleword{Ineq} rule is not
applicable to clause ${\Gamma \rightarrow \Delta \vee L}$, and so we have
${\Gamma \rightarrow \Delta \strin \S{v}}$. Since ${\Gamma\Grounding{t}
\subseteq \QueryLHS{t}}$, by \cref{lemma:lift-redundancy} we have ${\Gamma^i
\rightarrow \Delta^i \strin \GroundS{t}}$. Clearly, ${\Delta^i \vee l^i
\nequals r^i \gsucc{t} \Delta^i}$, and so \cref{lemma:all:true:before} implies
${\mdl{\Model{t}} \Delta^i}$, which is a contradiction.

\medskip

Assume ${L^i = l^i \nequals r^i}$ with ${l^i \gsucc{t} r^i}$.
\cref{lemma:inequality:stays} ensures ${\nmdl{(\Model{t}^{i-1})} l^i \nequals
r^i}$; hence, $l^i$ is reducible by $\Model{t}^{i-1}$ so, by the definition of
reducibility, a position $p$ and a generative clause ${C^j = \Gamma^j
\rightarrow \Delta^j \vee l^j \equals r^j}$ exist such that ${j < i}$ and
${\atPos{l^i}{p} = l^j}$. Due to ${j < i}$, we have ${l^i \nequals r^i
\gsucc{t} l^j \equals r^j \gsucc{t} \Delta^j}$. \cref{lemma:generative:clauses}
ensures ${\nmdl{\Model{t}} \Delta^j}$, and the definition of $\GroundS{t}$
ensures that a clause ${\Gamma' \rightarrow \Delta' \vee l' \equals r' \in
\S{v}}$ exists satisfying \cref{eq:generational:premise}, as in the first case.
By the assumption of \cref{theorem:completeness}, the \ruleword{Eq} rule is not
applicable to clauses \cref{eq:nonground:premise,eq:generational:premise}, and
so ${\Gamma \wedge \Gamma' \rightarrow \Delta \vee \Delta' \vee
\replPos{l}{r'}{p} \nequals r \strin \S{v}}$ holds. Let ${\Delta'' = \Delta^i
\vee \Delta^j \vee \replPos{l^i}{r^j}{p} \nequals r^i}$. We clearly have
${\Gamma\Grounding{t} \cup \Gamma'\Grounding{t} \subseteq \QueryLHS{t}}$, so by
\cref{lemma:lift-redundancy} we have ${\Gamma^i \wedge \Gamma^j \rightarrow
\Delta'' \strin \GroundS{t}}$. Since $\Star{\Model{t}}$ is a congruence, we
have ${\nmdl{\Model{t}} \replPos{l^i}{l^j}{p} \nequals r^i}$, and therefore
${\nmdl{\Model{t}} \Delta''}$ holds. Finally, $\gsucc{t}$ is a simplification
order, so ${l^i \nequals r^i \gsucc{t} \replPos{l^i}{l^j}{p}}$; together with
${l^i \equals r^i \gsucc{t} \Delta^i}$ and ${l^i \equals r^i \gsucc{t}
\Delta^j}$, we have ${l^i \equals r^i \gsucc{t} \Delta''}$. But then,
\cref{lemma:all:true:before} implies ${\mdl{\Model{t}} \Delta''}$, which is a
contradiction.
\end{proof}

\begin{lemma}\label{lemma:strengthening:true}
    For each clause ${\Gamma \rightarrow \Delta}$ with ${\Gamma \rightarrow
    \Delta \strin \GroundS{t}}$, we have ${\mdl{\Model{t}} \Delta}$.
\end{lemma}

\begin{proof}
Apply \cref{lemma:all:true:before} for ${i = n+1}$ and
\cref{lemma:interpretation:satisfies:clause}.
\end{proof}

\begin{lemma}\label{lemma:generative:no-reflexive-inequality}
    For each generative clause ${\Gamma^i \rightarrow \Delta^i \vee l^i
    \equals r^i}$, disjunction $\Delta^i$ does not contain a literal of the
    form ${s \nequals s}$.
\end{lemma}

\begin{proof}
For the sake of a contradiction, let us assume that clause ${C^i = \Gamma^i
\rightarrow \Delta^i \vee l^i \equals r^i \in \GroundS{t}}$ is generative and
that ${s \nequals s \in \Delta^i}$ holds for some term $s$. By the definition
of $\GroundS{t}$, a clause ${\Gamma' \rightarrow \Delta' \vee s' \nequals s'
\vee l' \equals r' \in \S{v}}$ exists such that
\begin{align}
    \Gamma'\Grounding{t} = \Gamma^i \subseteq \QueryLHS{t}, \quad \Delta'\Grounding{t} \cup \setsingle{s'\Grounding{t} \nequals s'\Grounding{t}} = \Delta^i, \quad s'\Grounding{t} = s, \quad l'\Grounding{t} = l^i, \quad \text{and} \quad r'\Grounding{t} = r^i.
\end{align}
By assumption of \cref{theorem:completeness}, the \ruleword{Ineq} rule is not
applicable to this clause, and so we have ${\Gamma' \rightarrow \Delta' \vee l'
\equals r' \strin \S{v}}$. Thus, we have ${\Gamma^i \rightarrow
\Delta'\Grounding{t} \vee l^i \equals r^i \strin \GroundS{t}}$, and so ${\Gamma
\rightarrow \Delta \in \GroundS{t}}$ holds for some ${\Gamma \subseteq
\Gamma^i}$ and some ${\Delta \subsetneq \Delta^i \cup \setsingle{l^i \equals
r^i}}$. Now \cref{lemma:generative:clauses} implies ${\nmdl{\Model{t}}
\Delta^i}$; moreover, by \cref{cond:rule:not:models}, we have
${\nmdl{(\Model{t}^{i-1})} \Delta^i \vee l^i \equals r^i}$. However, by
\cref{lemma:interpretation:satisfies:clause} we have ${\mdl{\Model{t}} \Gamma
\rightarrow \Delta}$. Now let $j$ be the index of clause ${\Gamma \rightarrow
\Delta}$ in the sequence of clauses from
\cref{sec:proofs:completeness:fragment:construction}; due to
${\Star{(\Model{t}^j)} \subseteq \Star{(\Model{t}^{i-1})}}$ and
\cref{lemma:inequality:stays}, we have ${\mdl{(\Model{t}^j)} \Gamma \rightarrow
\Delta}$. Since ${j < i}$, by the same argument we have
${\mdl{(\Model{t}^{i-1})} \Gamma \rightarrow \Delta}$. But then, ${\Delta
\subseteq \Delta^i \vee l^i \equals r^i}$ implies ${\mdl{(\Model{t}^{i-1})}
\Delta^i \vee l^i \equals r^i}$, which is a contradiction.
\end{proof}

\begin{lemma}\label{lemma:query:not:true}
    ${\nmdl{\Model{t}} \QueryLHS{t} \rightarrow \QueryRHS{t}}$.
\end{lemma}

\begin{proof}
For ${\mdl{\Model{t}} \QueryLHS{t}}$, note that \cref{cond:local:RHS} ensures
${\QueryLHS{t} \rightarrow A \strin \GroundS{t}}$, and so
\cref{lemma:strengthening:true} ensures ${\mdl{\Model{t}} A}$ for each atom ${A
\in \QueryLHS{t}}$.

For ${\nmdl{\Model{t}} \QueryRHS{t}}$, assume for the sake of a contradiction
that an atom ${A \in \QueryRHS{t}}$ exists such that ${\mdl{\Model{t}} A}$.
Then, a generative clause ${C^i = \Gamma^i \rightarrow \Delta^i \vee l^i
\equals r^i \in \GroundS{t}}$ and a position $p$ exist such that ${\atPos{A}{p}
= l^i}$; let ${\Delta = \Delta^i \vee l^i \equals r^i}$. Since $\gsucc{t}$ is a
simplification order and ${l^i \gsucc{t} r^i}$, we have ${A \gsucceq{t} l^i
\equals r^i}$; but then, since ${l^i \equals r^i \gsucc{t} \Delta^i}$, we have
${A \gsucceq{t} \Delta}$. We next consider an arbitrary literal ${l \eqorineq r
\in \Delta}$ with ${{\eqorineq} \in \setsingle{{\equals}, {\nequals}}}$ and ${l
\gsucceq{t} r}$; by the observations made thus far, ${A \gsucceq{t} l \eqorineq
r}$ holds. By \cref{cond:order:query}, one of the following holds.
\begin{enumerate}
    \item ${l \in \setsingle{ t, t'}}$. Moreover, since ${l \eqorineq r}$ is
    obtained by grounding a context literal, both $l$ and $r$ can be of the
    form $f(t)$ or $t'$. Together with ${l \gsucceq{t} r}$, we have ${l = r =
    t'}$. Now if ${l \eqorineq r}$ is ${t' \equals t'}$, then clause $C^i$ is
    not generative due to \cref{cond:rule:not:models}. Hence, the only
    remaining possibility is for ${l \eqorineq r}$ to be of the form ${t'
    \nequals t'}$; but then, clause $C^i$ is not generative by
    \cref{lemma:generative:no-reflexive-inequality}. Consequently, in either
    case we get a contradiction.
    
    \item ${l \equals r \in \QueryRHS{t}}$ where ${r = \TOP}$.
\end{enumerate}
Thus, the second point above holds for arbitrary ${l \eqorineq r \in \Delta}$,
and therefore we have ${\Delta \subseteq \QueryRHS{t}}$. But then, ${\Gamma^i
\subseteq \QueryLHS{t}}$ implies that ${\QueryLHS{t} \rightarrow \QueryRHS{t}
\strin \GroundS{t}}$ holds, which contradicts \cref{cond:local:query}.
\end{proof}

\subsection{Interpreting the Ontology $\Onto$}\label{sec:proofs:completeness:onto}

We now combine the rewrite systems $\Model{t}$ constructed in
\cref{sec:proofs:completeness:fragment} into a single rewrite system
$\EntireModel$, and we then show that $\Star{\EntireModel}$ satisfies
${\mdl{\EntireModel} \Onto}$ and ${\nmdl{\EntireModel} \InputQueryLHS
\rightarrow \InputQueryRHS}$.

\subsubsection{Unfolding the Context Structure}\label{sec:proofs:completeness:onto:unfolding}

We construct $\EntireModel$ by a partial induction over the terms in $\T$. We
define several partial functions: function $\VertexMapSymbol$ maps a term $t$
to a context ${\VertexMap{t} \in \V}$; functions $\Gamma$ and $\Delta$ assign
to a term $t$ a conjunction $\QueryLHS{t}$ and a disjunction $\QueryRHS{t}$,
respectively, of atoms; and function $\EntireModel$ maps each term into a model
fragment $\Model{t}$ for $t$, $\VertexMap{t}$, $\QueryLHS{t}$, and
$\QueryRHS{t}$.
\begin{enumerate}[label={M\arabic*.},ref={M\arabic*},leftmargin=2.5em]
    \item\label[condition]{def:model:base}
    For the base case, we consider the constant $c$.
    \begin{flalign}
        \VertexMap{c}   & = q                                                                               & \\
        \QueryLHS{c}    & = \InputQueryLHS\Grounding{c}                                                     &\label{eq:root:query:LHS} \\
        \QueryRHS{c}    & = \InputQueryRHS\Grounding{c}                                                     & \label{eq:root:query:RHS} \\
        \Model{c}       & = \text{the model fragment for } c, q, \QueryLHS{c}, \text{ and } \QueryRHS{c}    & \label{eqroot:model}
    \end{flalign}

    \item\label[condition]{def:model:step}
    For the inductive step, assume that $\VertexMap{t'}$ has already been
    defined, and consider an arbitrary function symbol ${f \in \F}$ such that
    $f(t')$ is irreducible by $\Model{t'}$. Let ${u = \VertexMap{t'}}$ and ${t
    = f(t')}$. We have two possibilities.
    \begin{enumerate}[label={M2.\alph*.},ref={M2.\alph*},leftmargin=2.8em]
        \item\label[condition]{def:model:step:normal}
        Term $t$ occurs in $\Model{t'}$. Then, term ${t = f(t')}$ was generated
        in $\Model{t'}$ by some ground clause ${C = \Gamma \rightarrow \Delta
        \vee L \in \GroundS{t'}}$ such that ${L \gsucc{t} \Delta}$ and $f(t')$
        occurs in $L$. By the definition of $\GroundS{t}$, then a clause ${C' =
        \Gamma' \rightarrow \Delta' \vee L' \in \S{u}}$ exists such that ${C =
        C'\Grounding{t'}}$ and $L'$ contains $f(x)$; moreover, ${L \gsucc{t'}
        \Delta}$ implies ${\Delta' \not\csucceq{u} L'}$. The \ruleword{Succ}
        and \ruleword{Core} rules are not applicable to $\D$, so we can choose
        a context ${v \in \V}$ such that ${\tuple{u,v,f} \in \E}$ and ${A
        \rightarrow A \strin \S{v}}$ for each ${A \in K_2}$, where $K_2$ is as
        in the \ruleword{Succ} rule. We define the following:
        \begin{flalign}
            \VertexMap{t}   & = v                                                                               & \\
            \QueryLHS{t}    & = \Star{\Model{t'}} \cap \GroundSu{t}                                             & \label{eq:context:query:LHS} \\
            \QueryRHS{t}    & = \GroundPr{t} \setminus \Star{\Model{t'}}                                        & \label{eq:context:query:RHS} \\
            \Model{t}       & = \text{the model fragment for } t, v, \QueryLHS{t}, \text{ and } \QueryRHS{t}    & \label{eq:context:model}
        \end{flalign}

        \item\label[condition]{def:model:step:abnormal}
        Term $t$ does not occur in $\Model{t'}$. Then, let ${\Model{t} =
        \setsingle{t \RuleSymbol c}}$, and we do not define any other functions
        for $t$.
    \end{enumerate}
\end{enumerate}
Finally, let $\EntireModel$ be the rewrite system defined by ${\EntireModel =
\bigcup_t \Model{t}}$.

\begin{lemma}\label{lemma:pairwise:compatible}
    The model fragments $\Model{c}$ and $\Model{t}$ constructed in lines
    \cref{eqroot:model} and \cref{eq:context:model} satisfy
    \cref{cond:local:query,cond:local:LHS,cond:local:RHS} in
    \cref{sec:proofs:completeness:fragment}.
\end{lemma}

\begin{proof}
The proof is by induction on the structure of terms ${t \in
\dom{\VertexMapSymbol}}$. For ${t = c}$,
\cref{cond:local:query,cond:local:LHS,cond:local:RHS} hold directly from
conditions \ref{theorem:completeness:query} through
\ref{theorem:completeness:LHS} of \cref{theorem:completeness}. We next assume
that the lemma holds for some term ${t' \in \dom{\VertexMapSymbol}}$, and we
consider an arbitrary term $t$ of the form ${t = f(t')}$; let ${u =
\VertexMap{t'}}$ and ${v = \VertexMap{t}}$.
\Cref{cond:local:RHS} holds because
${\QueryRHS{t} = \GroundPr{t} \setminus \Star{\Model{t'}}}$ due to \cref{eq:context:query:RHS},
and hence ${\QueryRHS{t} \subseteq \GroundPr{t}}$.
Before proceeding,
note that terms $t$ and $t'$ are irreducible by $\Model{t'}$ due to
\cref{def:model:step}; but then, since ${\QueryLHS{t} \subseteq
\Star{\Model{t'}}}$ holds by \cref{eq:context:query:LHS}, each each atom ${A_i
\in \Model{t'}}$ is generated by clause satisfying \cref{eq:pred:side:ground}
(where subscript $i$ does not necessarily indicate the position of the clause
in sequence of clauses from
\cref{sec:proofs:completeness:fragment:construction}). By the definition of
$\GroundS{t'}$, then there exists a clause satisfying
\cref{eq:pred:side:nonground}.
\begin{align}
    \Gamma_i \rightarrow \Delta_i \vee A_i      & \in \GroundS{t'}  && \text{with} \quad A_i \gsucc{t} \Delta_i \label{eq:pred:side:ground} \\
    \Gamma_i' \rightarrow \Delta_i' \vee A_i'   & \in \S{u}         && \Gamma_i = \Gamma_i'\Grounding{t'}, \quad \Delta_i = \Delta_i'\Grounding{t'}, \quad A_i = A_i'\Grounding{t'}, \quad \text{and} \quad \Delta_i' \not\csucceq{u} A_i' \label{eq:pred:side:nonground}
\end{align}

\smallskip

For \cref{cond:local:LHS}, consider an arbitrary atom ${A_i \in \QueryLHS{t}}$,
let \cref{eq:pred:side:ground} be the clause that generates $A_i$ in
$\Model{t'}$, and let \cref{eq:pred:side:nonground} be the corresponding
nonground clause. Since ${A_i \in \GroundSu{t}}$, atom $A_i'$ is of the form
$A_i''\sigma$, where $\sigma$ is the substitution from the \ruleword{Succ}
rule; but then, ${A_i'' \in K_2}$, where $K_2$ is as specified in the
\ruleword{Succ} rule. In \cref{def:model:step:normal} we chose $v$ so that the
\ruleword{Succ} rule is satisfied, and therefore ${A_i'' \rightarrow A_i''
\strin \S{v}}$; but then, since ${A_i''\Grounding{t} = A_i}$, we have ${A_i
\rightarrow A_i \strin \GroundS{t}}$, as required for \cref{cond:local:LHS}.

\smallskip

To prove that \cref{cond:local:query} holds as well, assume for the sake of a
contradiction that ${\QueryLHS{t} \rightarrow \QueryRHS{t} \strin \GroundS{t}}$
holds. We have ${\QueryRHS{t} \subseteq \GroundPr{t}}$ due to
\cref{eq:context:query:RHS}. Therefore, due to
\cref{cond:redundancy:strengthening} of \cref{def:redundancy}, set
$\GroundS{t}$ contains a clause
\begin{align}
    \bigwedge_{i=1}^m A_i \rightarrow \bigvee_{i=m+1}^{m+n} A_i \quad \text{with} \quad \setbuilder{A_i}{1 \leq i \leq m} \subseteq \QueryLHS{t} \quad \text{and} \quad \setbuilder{A_i}{m+1 \leq i \leq m+n} \subseteq \QueryRHS{t} \subseteq \GroundPr{t}. \label{eq:pred:main:ground}
\end{align}
By the definition of $\GroundS{t}$, set $\S{v}$ contains a clause
\begin{align}
    \bigwedge_{i=1}^m A_i' \rightarrow \bigvee_{i=m+1}^{m+n} A_i' \quad \text{where} \quad A_i = A_i'\Grounding{t} \text{ for } 1 \leq i \leq m+n \quad \text{and} \quad A_i' \in \PRtriggers \text{ for } m+1 \leq i \leq m+n. \label{eq:pred:main:nonground}
\end{align}
Now each $A_i$ with ${1 \leq i \leq m}$ is generated by a ground clause
\cref{eq:pred:side:ground}, and the latter is obtained from the corresponding
nonground clause \cref{eq:pred:side:nonground}. The \ruleword{Pred} rule is not
applicable to \cref{eq:pred:main:nonground,eq:pred:side:nonground} so
\cref{eq:pred:conclusion:nonground} holds; together with
\cref{lemma:lift-redundancy}, this ensures \cref{eq:pred:conclusion:ground}.
\begin{align}
    \bigwedge_{i=1}^m \Gamma_i' \rightarrow \bigvee_{i=1}^{m}\Delta_i' \vee \bigvee_{i=m+1}^{m+n} A_i'\sigma & \strin \S{u} \quad \text{for } \sigma = \subst{x \mapsto f(x), \; y \mapsto x} \label{eq:pred:conclusion:nonground} \\
    \bigwedge_{i=1}^m \Gamma_i \rightarrow \bigvee_{i=1}^{m}\Delta_i \vee \bigvee_{i=m+1}^{m+n} A_i          & \strin \GroundS{t'} \label{eq:pred:conclusion:ground}
\end{align}
By \cref{lemma:generative:clauses}, we have ${\nmdl{\Model{t'}} \Delta_i}$; and
\cref{eq:context:query:RHS} ensures that ${\nmdl{\Model{t'}} \QueryRHS{t}}$,
and so ${\nmdl{\Model{t'}} A_i}$ for each ${m+1 \leq i \leq m+n}$; however,
this contradicts \cref{eq:pred:conclusion:ground} and
\cref{lemma:strengthening:true}.
\end{proof}

\subsubsection{Termination, Confluence, and Compatibility}

\begin{lemma}\label{lemma:church-rosser}
    The rewrite system $\EntireModel$ is Church-Rosser.
\end{lemma}

\begin{proof}
We show that $\EntireModel$ is terminating and left-reduced, and thus
Church-Rosser. In the proof of the former, we use a total simplification order
$\msucc$ on all ground $\F$- and $\P$-terms defined as follows. We extend the
precedence $\funsucc$ from \cref{def:context:term:order} to all $\F$- and
$\P$-symbols in an arbitrary way, but ensuring that constant $\TOP$ is smallest
in the order; then, let $\msucc$ be a \emph{lexicographic path order}
\cite{baader98term} over such $\funsucc$. It is well known that such $\msucc$
is a simplification order, and that it satisfies the following properties for
each $\F$-term $t$ with predecessor $t'$ (if one exists), all function symbols
${f,g \in \F}$, and each $\P$-term $A$:
\begin{itemize}
    \item ${f(t) \msucc t \msucc t'}$,
    \item ${f \funsucc g}$ implies ${f(t) \msucc g(t)}$, and
    \item ${A \msucc \TOP}$.
\end{itemize}
Thus, \cref{cond:term:order:f-term:1,cond:term:order:f-term:2} of
\cref{def:context:term:order} and the manner in which context orders are
grounded in \cref{sec:proofs:completeness:fragment:order} clearly ensure that,
for each $\F$-term ${t \in \dom{\VertexMapSymbol}}$ and for all terms $s_1$ and
$s_2$ from the $\F$-neighbourhood of $t$ with ${s_1 \gsucc{t} s_2}$, we have
${s_1 \msucc s_2}$.

\medskip

We next show that $\EntireModel$ is terminating by arguing that each rule in
$\EntireModel$ is embedded in $\msucc$. To this end, consider an arbitrary rule
${\Rule{l}{\EntireModel}{r}}$. Clearly, a term ${t \in \dom{\EntireModel}}$
exists such that ${\Rule{l}{\Model{t}}{r}}$. This rule is obtained from a head
${l \equals r}$ of a clause in $\GroundS{t}$, and \cref{cond:rule:order} of the
definition of $\Model{t}$ ensures that ${l \gsucc{t} r}$. Moreover, ${l \equals
r}$ is obtained by grounding a context literal with $\Grounding{t}$, so we have
the following possible forms of ${l \equals r}$.
\begin{itemize}
    \item Terms $l$ and $r$ are both from the $\F$-neighbourhood of $t$. Then,
    ${l \gsucc{t} r}$ implies ${l \msucc r}$.

    \item We have ${l \equals r = A \equals \TOP}$ for $A$ a $\P$-term. Then,
    ${A \msucc \TOP}$ since $\TOP$ is smallest in $\funsucc$.
\end{itemize}

\medskip

We next show that $\EntireModel$ is left-reduced. For the sake of a
contradiction, assume that a rule ${\Rule{l}{\EntireModel}{r}}$ exists such
that $l$ is reducible by ${\EntireModel' = \EntireModel \setminus \setsingle{ l
\RuleSymbol r }}$. Let $p$ be the `deepest' position at which some rule in
$\EntireModel'$ reduces $l$ (i.e., no rule in $\EntireModel'$ reduces $l$ at
position below $p$), and let ${\Rule{l'}{\EntireModel'}{r'}}$ be the rule that
reduces $l$ at position $p$; thus, ${l' = \atPos{l}{p}}$. By the definition of
$\EntireModel$, we have ${\Rule{l'}{\Model{t}}{r'}}$ where $t$ can be as
follows.
\begin{itemize}
    \item Term $t$ is handled in \cref{def:model:step:normal}. Then ${l'
    \RuleSymbol r'}$ is generated by an equality ${l' \equals r'}$ in the head
    of a generative clause, and so $l'$ is of the form $f(t)$. Thus, $f(t)$ is
    reducible by $\Model{t}$, which contradicts \cref{def:model:step} from the
    construction of $\EntireModel$.

    \item Term $t$ is handled in \cref{def:model:step:abnormal}. Then ${l' =
    t}$; moreover, $\EntireModel'$ does not contain $t$ by the construction of
    $\EntireModel$, which contradicts the assumption that
    ${\Rule{l'}{\EntireModel'}{r'}}$. \qedhere
\end{itemize}
\end{proof}

\begin{lemma}\label{lemma:compatibility}
    For each term $t$, each ${f \in \F}$, and each atom ${A \in \GroundSu{t}
    \cup \GroundPr{f(t)}}$ such that ${A \in \Star{\EntireModel}}$ and all
    $\F$-terms in $A$ are irreducible by $\EntireModel$, we have ${A \in
    \Star{\Model{t}}}$.
\end{lemma}

\begin{proof}
Let $t$ be a term, let ${f \in \F}$ be a function symbol, and let ${A \in
\GroundSu{t} \cup \GroundPr{f(t)}}$ be an atom such that all $\F$-terms in $A$
are irreducible by $\EntireModel$; the latter ensures
${\Rule{A}{\EntireModel}{\TOP}}$. We next consider the possible forms of $A$.

Assume ${A \in \GroundSu{t}}$. By the definition of $\GroundSu{t}$ in
\cref{eq:ground:SU} and the fact that $\SUtriggers$ contains only atoms of the
form $B(x)$, $S(x,y)$, and $S(y,x)$, atom $A$ can be of the form $B(t)$,
$S(t,t')$, or $S(t',t)$, for $t'$ the predecessor of $t$ (if one exists). By
the form of the generative clauses, we clearly have ${A \in \Star{\Model{t}}}$
or ${A \in \Star{\Model{t'}}}$. Now assume ${A \in \Star{\Model{t'}}}$. Due to
${A \in \GroundSu{t}}$ and the definition of $\QueryLHS{t}$ in
\cref{eq:context:query:LHS}, we have ${A \in \QueryLHS{t}}$.
\cref{lemma:query:not:true} ensures that ${\nmdl{\Model{t}} \QueryLHS{t}
\rightarrow \QueryRHS{t}}$. But then, we have ${A \in \Star{\Model{t}}}$, as
required.

Assume ${A \in \GroundPr{f(t)}}$. By the definition of $\GroundPr{f(t)}$ in
\cref{eq:ground:PR} and the fact that $\PRtriggers$ contains only atoms of the
form $B(y)$, $S(y,x)$, and $S(x,y)$, atom $A$ can be of the form $B(t)$,
$S(t,f(t))$, or $S(f(t),t)$. By the form of the generative clauses, we clearly
have ${A \in \Star{\Model{t}}}$ or ${A \in \Star{\Model{f(t)}}}$. Assume for
the sake of a contradiction that ${A \not\in \Star{\Model{t}}}$, but ${A \in
\Star{\Model{f(t)}}}$. Due to ${A \in \GroundPr{f(t)}}$ and the definition of
$\QueryRHS{f(t)}$ in \cref{eq:context:query:RHS}, we have ${A \in
\QueryRHS{f(t)}}$; due to \cref{lemma:query:not:true}, we have
${\nmdl{\Model{f(t)}} \QueryLHS{f(t)} \rightarrow \QueryRHS{f(t)}}$; therefore,
we have ${A \not\in \Star{\Model{f(t)}}}$, which is a contradiction.
\end{proof}

\begin{lemma}\label{lemma:DL:literals:preserved}
    Let $s_1$ and $s_2$ be DL-terms, and let $\tau$ be a substitution
    irreducible by $\EntireModel$ such that $s_1\tau$ and $s_2\tau$ are ground
    and each $\tau(z_i)$ (if defined) is in the $\F$-neighbourhood of
    $\tau(x)$. Then, for ${{\eqorineq} \in \setsingle{{\equals}, \; {\nequals}}}$,
    if ${\mdl{\Model{\tau(x)}} s_1\tau \eqorineq s_2\tau}$, then
    ${\mdl{\EntireModel} s_1\tau \eqorineq s_2\tau}$.
\end{lemma}

\begin{proof}
Let $s_1$ and $s_2$ and $\tau$ be as stated above, let ${t = \tau(x)}$, and let
$t'$ be the predecessor of $t$ (if one exists). Since $t$ is irreducible by
$\EntireModel$, rewrite system $\Model{t}$ has been defined in
\cref{sec:proofs:completeness:onto:unfolding}. We next consider the possible
forms of $\eqorineq$.
\begin{itemize}
    \item Assume ${{\eqorineq} = {\equals}}$. But then, ${\Model{t} \subseteq
    \EntireModel}$ and ${\mdl{\Model{t}} s_1\tau \equals s_2\tau}$ imply
    ${\mdl{\EntireModel} s_1\tau \equals s_2\tau}$.

    \item Assume ${{\eqorineq} = {\nequals}}$. Let $s_1'$ and $s_2'$ be the normal
    forms of $s_1\tau$ and $s_2\tau$, respectively, w.r.t.\ $\Model{t}$. Due to
    the shape of DL-literals, $s_1$ and $s_2$ can be of the form $f(x)$ or
    $z_i$; therefore, $s_1\tau$ and $s_2\tau$ are of the form $f(t)$ or $t'$.
    Term $t$ is irreducible by $\EntireModel$, and thus $t'$ is irreducible by
    $\EntireModel$ as well. Furthermore, due to the shape of context terms, the
    only rewrite system where $f(t)$ could occur on the left-hand side of a
    rewrite rule is $\Model{t}$. Consequently, $f(t)$ is irreducible by
    $\EntireModel$ as well. But then, $s_1'$ and $s_2'$ are the normal forms of
    $s_1\tau$ and $s_2\tau$, respectively, w.r.t.\ $\EntireModel$; thus,
    ${\mdl{\EntireModel} s_1' \nequals s_2'}$, and thus ${\mdl{\EntireModel}
    s_1\tau \nequals s_2\tau}$ holds, as required. \qedhere
\end{itemize}
\end{proof}

\subsubsection{The Completeness Claim}

\begin{lemma}\label{lemma:models:onto}
    For each DL-clause ${\Gamma \rightarrow \Delta \in \Onto}$, we have
    ${\mdl{\EntireModel} \Gamma \rightarrow \Delta}$.
\end{lemma}

\begin{proof}
Consider an arbitrary DL-clause ${\Gamma \rightarrow \Delta \in \Onto}$ of the
following form:
\begin{align}
    \textstyle \bigwedge_{i=1}^n A_i \rightarrow \Delta \label{eq:hyper:main:nonground}
\end{align}
Let $\tau'$ be an arbitrary substitution such that ${\Gamma\tau' \rightarrow
\Delta\tau'}$ is ground, and let $\tau$ be the substitution obtained from
$\tau'$ by replacing each ground term with its normal form w.r.t.\
$\EntireModel$. Since $\Star{\EntireModel}$ is a congruence, we have
${\mdl{\EntireModel} \Gamma\tau' \rightarrow \Delta\tau'}$ \iff
${\mdl{\EntireModel} \Gamma\tau \rightarrow \Delta\tau}$. We next assume that
${\mdl{\EntireModel} \Gamma\tau}$, and we show that ${\mdl{\EntireModel}
\Delta\tau}$ holds as well.

Consider an arbitrary atom ${A_i \in \Gamma}$. By the definition of DL-clauses,
$A_i$ is of the form $B(x)$, $S(x,z_j)$, or $S(z_j,x)$. Substitution $\tau$ is
irreducible by $\EntireModel$, and so all $\F$-terms in $A_i\tau$ are
irreducible by $\EntireModel$; but then, ${A_i\tau \in \Star{\EntireModel}}$
clearly implies ${\Rule{A_i\tau}{\EntireModel}{\TOP}}$. Each such rule is
obtained from a generative clause so $A_i\tau$ is of the form $B(t)$,
$S(t,f(t))$, $S(f(t),t)$, $S(t,t')$, or $S(t',t)$, where ${t = \tau(x)}$ and
$t'$ is the predecessor of $t$ (if it exists). We next prove that ${A_i\tau \in
\GroundSu{t} \cup \GroundPr{f(t)}}$ holds by considering the possible forms of
$A_i$.
\begin{itemize}
    \item ${A_i = B(x)}$, so ${A_i\tau = B(t)}$. Then, we have ${B(x) \in
    \SUtriggers}$, which implies that ${B(t) \in \GroundSu{t}}$ holds.

    \item ${A_i = S(x,z_j)}$, so ${A_i\tau}$ is of the form $S(t,t')$ or
    $S(t,f(t))$. Then, we have ${S(x,y) \in \SUtriggers}$, which implies that
    ${S(t,t') \in \GroundSu{t}}$ holds; moreover, we have ${S(y,x) \in
    \PRtriggers}$, which implies that ${S(t,f(t)) \in \GroundPr{f(t)}}$ holds.

    \item ${A_i = S(z_j,x)}$, so ${A_i\tau}$ is of the form $S(t',t)$ or
    $S(f(t),t)$. Then, we have ${S(y,x) \in \SUtriggers}$, which implies that
    ${S(t',t) \in \GroundSu{t}}$ holds; moreover, we have ${S(x,y) \in
    \PRtriggers}$, which implies that ${S(f(t),t) \in \GroundPr{f(t)}}$ holds.
\end{itemize}
\cref{lemma:compatibility} then implies ${A_i\tau \in \Model{t}}$, and so
$\GroundS{t}$ contains a generative clause of the form
\cref{eq:hyper:generative:ground}. Now let ${v = \VertexMap{t}}$; by the definition
of $\GroundS{t}$, set $\S{v}$ contains a clause of the form
\cref{eq:hyper:generative:nonground}.
\begin{align}
    \Gamma_i \rightarrow \Delta_i \vee A_i      & \text{ with } A_i  \gsucc{t} \Delta_i \text{ and } \Gamma_i \subseteq \QueryLHS{t} \label{eq:hyper:generative:ground} \\
    \Gamma_i' \rightarrow \Delta_i' \vee A_i'   & \text{ with } \Delta_i'  \not\nsucceq{v} A_i' \text{ and } \Gamma_i'\Grounding{t} = \Gamma_i, \; \Delta_i'\Grounding{t} = \Delta_i, \text{ and } A_i'\Grounding{t} = A_i \label{eq:hyper:generative:nonground}
\end{align}
The \ruleword{Hyper} rule is not applicable to
\cref{eq:hyper:main:nonground,eq:hyper:generative:nonground}, and therefore
\cref{eq:hyper:conclusion:nonground} holds, where $\sigma$ is the substitution
obtained from $\tau$ by replacing each occurrence of $t$ (possibly nested in
another term) with $x$. Finally, \cref{lemma:lift-redundancy} ensures that
\cref{eq:hyper:conclusion:ground} holds as well.
\begin{align}
    \bigwedge_{i=1}^n \Gamma_i' \rightarrow \Delta\sigma \vee \bigvee_{i=1}^n \Delta_i' & \strin \S{v} \label{eq:hyper:conclusion:nonground} \\
    \bigwedge_{i=1}^n \Gamma_i \rightarrow \Delta\tau \vee \bigvee_{i=1}^n \Delta_i     & \strin \GroundS{t} \label{eq:hyper:conclusion:ground}
\end{align}
Now \cref{eq:hyper:conclusion:ground} and \cref{lemma:strengthening:true} imply
${\mdl{\Model{t}} \Delta\tau \vee \bigvee_{i=1}^n \Delta_i}$, but
\cref{lemma:generative:clauses} implies ${\nmdl{\Model{t}} \Delta_i}$;
therefore, we have ${\mdl{\Model{t}} \Delta\tau}$. Finally,
\cref{lemma:DL:literals:preserved} ensures ${\mdl{\EntireModel} \Delta\tau}$,
as required.
\end{proof}

\begin{lemma}\label{lemma:do:not:model:query}
    ${\nmdl{\EntireModel} \InputQueryLHS \rightarrow \InputQueryRHS}$.
\end{lemma}

\begin{proof}
The claim clearly follows from ${\nmdl{\EntireModel} \QueryLHS{c} \rightarrow
\QueryRHS{c}}$. Note that \cref{lemma:query:not:true} ensures
${\nmdl{\Model{c}} \QueryLHS{c} \rightarrow \QueryRHS{c}}$; thus,
${\mdl{\Model{c}} \QueryLHS{c}}$ and ${\nmdl{\Model{c}} \QueryRHS{c}}$. The
former observation and \cref{lemma:DL:literals:preserved} ensure that
${\mdl{\EntireModel} \QueryLHS{c}}$ holds. Moreover, for each atom ${B(x) \in
\InputQueryRHS}$, \cref{def:triggers} ensures ${B(y) \in \PRtriggers}$; thus,
for each ${f \in \F}$, we have ${B(c) \in \GroundPr{f(c)}}$, and so the
contrapositive of \cref{lemma:compatibility} ensures ${\nmdl{\EntireModel}
B(c)}$. Thus, ${\nmdl{\EntireModel} \QueryRHS{c}}$ holds, as required.
\end{proof}

\section{Proof of \cref{prop:complexity}}

\propcomplexity*

\begin{proof}
The number $\wp$ of context clauses that can be generated using the symbols in
$\Onto$ is at most exponential in the size of $\Onto$, and the number $m$ of
clauses participating in each inference is linear in the size of $\Onto$.
Hence, with $k$ contexts, the number of inferences is bounded by ${(k \cdot
\wp)^m}$; if $k$ is at most exponential in the size of $\Onto$, the number of
inferences is exponential as well. Thus, if at most exponentially many contexts
are introduced, our algorithm runs in exponential time.
\end{proof}

\section{Proof of \cref{prop:worst-case-el}}

\propelh*

\begin{proof}
Consider an \ELH ontology that is transformed into a set $\Onto$ of DL-clauses
as specified in \cref{sec:preliminaries}, and consider a query of the form
${B_1(x) \rightarrow B_2(x)}$. Due to the form of the query, the core of $q$ is
initialised to $B_1(x)$.

We first consider applying algorithm \ref{alg:initialize-D}--\ref{alg:read} to
$\Onto$ with the cautious strategy and the eager application of the
\ruleword{Hyper} rule. By induction on the application of the rules from
\cref{table:rules}, we next show that each context clause derived by the rules
is of the form \cref{eq:worst-case-el:1}--\cref{eq:worst-case-el:5} and that
the core of each context is of the form $B(x)$.
\begin{align}
    \top        & \rightarrow B(x) \label{eq:worst-case-el:1} \\
    \top        & \rightarrow S(x,f(x)) \label{eq:worst-case-el:2} \\
    \top        & \rightarrow B(f(x)) \label{eq:worst-case-el:3} \\
    S(y,x)      & \rightarrow B(y) \label{eq:worst-case-el:4} \\
    S_1(y,x)    & \rightarrow S_2(y,x) \label{eq:worst-case-el:5}
\end{align}
In particular, in step \ref{alg:apply-rules} we can perform the following
inferences, with the specified correspondence to the completion rules CR1--CR4
and CR10 by \citeA{babl05}.
\begin{itemize}
    \item The core of each context is of the form $B(x)$, so the
    \ruleword{Core} rule introduces a clause of the form form
    \eqref{eq:worst-case-el:1}. This corresponds to way in which \citeA{babl05}
    initialise their mappings.
    
    \item The \ruleword{Hyper} rule can be applied to a DL-clause of type DL1.
    All other clauses participating in the inference are of the form
    \eqref{eq:worst-case-el:1}, so the result is of the form
    \eqref{eq:worst-case-el:1}. Such an inference corresponds to the completion
    rules CR1 or CR2.

    \item The \ruleword{Hyper} rule can be applied to a DL-clause of type DL2.
    The other clause participating in the inference is of the form
    \eqref{eq:worst-case-el:1}, so the result is of the form
    \eqref{eq:worst-case-el:2} or \eqref{eq:worst-case-el:3}. Moreover,
    function symbol $f$ occurs in $\Onto$ in exactly one pair of clauses DL2,
    and the \ruleword{Hyper} rule is applied eagerly; thus, whenever $f$ occurs
    in a context in a clause of the form \eqref{eq:worst-case-el:2}, it also
    occurs in a clause of the form \eqref{eq:worst-case-el:3}. Now the
    \ruleword{Succ} rule can be applied to the function symbol $f$, in which
    case the cautious strategy thus returns a context whose core is of the form
    $B(x)$. All of these inferences correspond to the completion rule CR3.

    \item The \ruleword{Hyper} rule can be applied to a DL-clause of type DL3.
    The two other clauses participating in the inference are of the form
    \eqref{eq:worst-case-el:4} and \eqref{eq:worst-case-el:1}, so the result is
    of the form \eqref{eq:worst-case-el:4}; the \ruleword{Pred} rule can then
    be applied to the latter clause, producing a clause of the form
    \eqref{eq:worst-case-el:1}. Such a pair of inferences corresponds to the
    completion rule CR4.
    
    \item The \ruleword{Hyper} rule can be applied to a DL-clause of type DL5.
    The other clause participating in the inference is of the form
    \eqref{eq:worst-case-el:4}, so the result is of the form
    \eqref{eq:worst-case-el:4} as well; the \ruleword{Pred} rule can then be
    applied to the latter clause, producing a clause of the form
    \eqref{eq:worst-case-el:2}. Such a pair of inferences corresponds to the
    completion rule CR10.
\end{itemize}
One can show in an analogous way that each inference of the calculus by
\citeA{babl05} corresponds to one or more inferences of our calculus.
Furthermore, it is clear that our algorithm runs in polynomial time.

\bigskip

We next consider applying algorithm \ref{alg:initialize-D}--\ref{alg:read} to
$\Onto$ with the eager strategy. One can show that the core of each
context is of the form $A(x)$, $R(y,x)$, or ${A(x) \wedge R(y,x)}$, and that
context can contain clauses of the form
\cref{eq:worst-case-el:1}--\cref{eq:worst-case-el:7}.
\begin{align}
    \top    & \rightarrow S_2(y,x) \label{eq:worst-case-el:6} \\
    \top    & \rightarrow B(y) \label{eq:worst-case-el:7}
\end{align}
The proof is analogous to the case of the cautions strategy (without
correspondence to the completion rules) so we omit the details for the sake of
brevity. The only minor difference is that, if an application of the the
\ruleword{Pred} to contexts $u$ and $v$ introduces a clause of the form
\eqref{eq:worst-case-el:2} in $u$, then the \ruleword{Succ} rule does not
become applicable to $u$ since the precondition of the \ruleword{Succ} rule is
still satisfied by $v$. Thus, the \ruleword{Succ} rule never introduces
contexts whose cores contain conjunctions of binary atoms. Thus, if $\Onto$
contains $k_1$ unary and $k_2$ binary predicates, the number of contexts is
bounded by $O(k_1 \cdot k_2)$, and each context can contain at most ${k_1 + k_2
+ k_1 \cdot k_2}$ clauses. All rules can be applied in polynomial time, so the
algorithm runs in polynomial time.
\end{proof}

\end{document}